\documentclass[twoside]{article}

\usepackage[accepted]{aistats2023}
\usepackage[round]{natbib}

\usepackage[utf8]{inputenc} 
\usepackage[T1]{fontenc}    
\usepackage{hyperref}       
\usepackage{url}            
\usepackage{booktabs}       
\usepackage{amsfonts}       
\usepackage{nicefrac}       
\usepackage{microtype}      
\usepackage{xcolor}         

\usepackage{wrapfig}

\usepackage{multirow}

\usepackage{amsmath}
\usepackage[normalem]{ulem}
\usepackage{algorithm}

\usepackage{algpseudocode}
\usepackage{enumitem}   
\usepackage{xspace}   

\usepackage[font=small,labelfont=bf]{caption}
\usepackage{subcaption}
\usepackage[percent]{overpic}
\usepackage{amssymb, amsmath}
\usepackage{amsthm}
\usepackage{stmaryrd}
\usepackage{caption}
\usepackage{bm}

\usepackage{color} 

\usepackage{dsfont}

\definecolor{airforceblue}{rgb}{0.36, 0.54, 0.66}

\usepackage{todonotes}
\newcommand{\transpose}{{\mkern-1.5mu\mathsf{T}}}
\newcommand{\EE}{\mathbb{E}}
\newcommand{\EEE}[1]{\mathbb{E}_{#1}}

\newcommand{\RR}{\mathbb{R}}
\newcommand{\calO}{\mathcal{O}}
\newcommand{\calD}{\mathcal{D}}
\newcommand{\calS}{\mathcal{S}}
\newcommand{\calF}{\mathcal{F}}

\newcommand{\1}{\mathbf{1}}
\newcommand{\indic}{\mathds{1}}

\newtheorem{theorem}{Theorem}
\newtheorem{assumption}{Assumption}

\newtheorem{proposition}{Proposition}
\newtheorem{lemma}{Lemma}
\newtheorem{corollary}{Corollary}
\newtheorem{example}{Example}

\newtheorem{property}{Property}

\newcommand{\overbar}[1]{\mkern 1.5mu\overline{\mkern-1.5mu#1\mkern-1.5mu}\mkern 1.5mu}

\makeatletter
\newcommand{\pushright}[1]{\ifmeasuring@#1\else\omit\hfill$\displaystyle#1$\fi\ignorespaces}

\newcommand{\pushleft}[1]{\ifmeasuring@#1\else\omit$\displaystyle#1$\hfill\fi\ignorespaces}

\definecolor{linkcolor}{RGB}{83,83,182}

\hypersetup{
    colorlinks=true,
    citecolor=linkcolor,
    linkcolor=linkcolor,
    urlcolor=linkcolor
}

\begin{document}

%

%

\twocolumn[

\aistatstitle{Refined Convergence and Topology Learning for \\ Decentralized SGD with Heterogeneous Data}

\aistatsauthor{ Batiste Le Bars \And Aurélien Bellet \And  Marc Tommasi }
\aistatsaddress{ Univ. Lille, Inria, CNRS, Centrale Lille,
UMR 9189, CRIStAL, F-59000 Lille}
\aistatsauthor{ Erick Lavoie \And Anne-Marie Kermarrec}
\aistatsaddress{ Université de Bâle, Bâle, Switzerland \And EPFL, Lausanne, Switzerland}]

\begin{abstract}
  \looseness=-1 One of the key challenges in decentralized and federated learning is to design algorithms that efficiently deal with highly heterogeneous data distributions across agents. In this paper, we revisit the analysis of the popular Decentralized Stochastic Gradient Descent algorithm (D-SGD) under data heterogeneity. We exhibit the key role played by a new quantity, called \emph{neighborhood heterogeneity}, on the convergence rate of D-SGD. By coupling the communication topology and the heterogeneity, our analysis sheds light on the poorly understood interplay between these two concepts. We then argue that neighborhood heterogeneity provides a natural criterion to learn data-dependent topologies that reduce (and can even eliminate) the otherwise detrimental effect of data heterogeneity on the convergence time of D-SGD. For the important case of classification with label skew, we formulate the problem of learning such a good topology as a tractable optimization problem that we solve with a Frank-Wolfe algorithm. As illustrated over a set of simulated and real-world experiments, our approach provides a principled way to design a sparse topology that balances the convergence speed and the per-iteration communication costs of D-SGD under data heterogeneity.
\end{abstract}

\section{Introduction}

Decentralized and federated learning methods allow training from data stored
locally by
several
 agents (nodes) without exchanging raw data, in line with the increasing
 demand for more privacy-preserving algorithms \citep{kairouz2021advances}. One
of the key challenges in decentralized learning is to deal with data
 heterogeneity: as each agent collects its own data,
 local datasets typically exhibit different distributions. In this work, we
 study this challenge in the context of
fully decentralized learning algorithms, 
which provide a scalable and robust alternative to server-based
approaches
\citep{colin2016gossip,lian2017can,koloskova2019decentralized,koloskova20}. 
%
Fully decentralized optimization algorithms, such as
the celebrated Decentralized SGD (D-SGD) 
\citep{lian2017can,Lian2018,koloskova20},
operate on a graph representing the communication topology, i.e. which pairs
of nodes exchange information with each other. The connectivity of the
topology then rules a trade-off between the convergence rate and the per-iteration
communication complexity of fully decentralized algorithms
\citep{wang2019matcha}.
Choosing a good topology for fully decentralized machine learning is therefore an important question, and remains a largely open problem in the presence of data heterogeneity.


Until recently, the impact of the communication topology on the convergence
 was believed to be mainly characterized by its spectral gap: a large spectral gap indicating good connectivity and thus faster convergence. 
Focusing solely on the connectivity of the topology
%
has however shown to be insufficient, even when we have identically distributed data \citep{neglia2020decentralized,vogels2022beyond}. %
In the heterogeneous setting, \cite{bellet2021d} 
notably observe that the choice of topology has a large influence, beyond its spectral gap, on the convergence speed of D-SGD. 
However, these empirical observations are not supported by any theory.

\looseness=-1 In this work, we fill the theoretical gap that currently exists on
 these questions. We focus on D-SGD \citep{lian2017can,Lian2018,koloskova20}, which is arguably the most popular decentralized optimization algorithm in the context of machine learning due to its good properties inherited from centralized SGD.
 In particular, D-SGD has been praised for its computational scalability \citep{lin2021quasi}, its applicability to training deep neural networks at scale \citep{ying2021exponential,kong2021consensus}, and the good generalization guarantees that it provides \citep{sun2021stability, zhu2022topology}. 
 
\looseness=-1 Our first contribution is a refined convergence
 analysis of D-SGD which introduces a new quantity, called \emph{neighborhood
 heterogeneity}, that couples the topology and the local data
 distributions.
Neighborhood heterogeneity essentially measures the expected distance between
 the \emph{global gradient} and the \emph{aggregated gradients in the neighborhood} of nodes.
 Our results demonstrate that the impact of the topology on the convergence
 rate of D-SGD, for both convex and non-convex objectives, does not only depend on its connectivity (i.e., spectral gap): it also depends on its capacity to compensate the heterogeneity of local data distributions at the neighborhood level.
 This new perspective allows to avoid the restrictive assumption of bounded
heterogeneity used in previous work \citep{lian2017can,Lian2018,tang2018d,assran2019stochastic,koloskova20,ying2021exponential}.

Our second contribution deals with the problem of learning a good 
\emph{data-dependent} topology, going beyond prior work which focused mainly on
 optimizing the spectral gap \citep{boyd2004fastest, boyd2006randomized,wang2019matcha}. 
We argue that neighborhood heterogeneity provides a natural objective
and show that it can be effectively optimized in practice in the important
case of classification with label distribution heterogeneity across
nodes (\emph{label skew}) \citep{kairouz2021advances,quagmire,bellet2021d}. We
solve the resulting
 problem using a Frank-Wolfe algorithm \citep{Frank1956AnAF,jaggi2013revisiting}, allowing us to track the quality of the learned topology as new edges are added in a greedy manner.
Our results imply that we can approximately minimize neighborhood heterogeneity up to a fixed additive error with a topology whose maximum degree is constant in the number of nodes.
To the best of our knowledge, our work is the first to learn the graph topology for decentralized learning in a way that (i) is data-dependent, (ii) controls communication costs, and (iii) optimizes the convergence rate of D-SGD.
We illustrate the usefulness of our approach in simulated and real data
experiments with linear and deep models.

\section{Related Work}
\label{sec:related}

\textbf{Consensus vs personalized objectives.}
In this work, we study the consensus problem which aims to learn a \emph{single} model that minimizes the average of the local objectives (see Eq.~\ref{eq:obj}).
Another line of research tackles the problem of heterogeneity in decentralized
learning through personalization \citep{koppel2017proximity,vanhaesebrouck2017decentralized,zantedeschi2020fully,Marfoq2021a,perso_even}. In that setting, each agent aims to learn a \emph{personalized} model that minimizes its own (expected) local objective.
It is thus natural and desirable to connect nodes that have similar data distributions. In contrast, our results show that for the consensus problem, the topology
should connect nodes that are different so that local neighborhoods are representative of the global distribution. We emphasize that personalization and consensus are relevant to different use cases and can be considered as orthogonal to each other. 

\looseness=-1 \textbf{Algorithmic improvements to decentralized SGD.}
Significant work has been devoted to extensions of D-SGD. We can mention those based on momentum 
\citep{assran2019stochastic,gao2020periodic,lin2021quasi,yuan2021decentlam}, cross-gradient aggregations \citep{cross_gradient}, gradient tracking \citep{koloskova2021improved} and bias correction (or variance reduction) \citep{tang2018d,yuan2020influence,yuan2021removing,huang2021improving}. Many of these schemes are able to reduce the order of the term that depends on data heterogeneity 
but remain impacted by strong heterogeneous scenarios.
We stress that the above line of research is complementary to ours as it is
based on modifications of the D-SGD algorithm (which often requires
additional computation and/or communication). In contrast, our work does not
modify the algorithm: we provide a refined analysis and a method to learn the
topology. We believe that our results can be combined with the above
algorithmic improvements, but leave such extensions for future work.

\textbf{Good topologies for decentralized learning.}
There is a long line of research on choosing a good topology (e.g., expanders or
exponential graphs) \citep{chow2016expander, nedic2018network,
 ying2021exponential}, or learning it to maximize the spectral gap 
\citep{boyd2004fastest, boyd2006randomized,wang2019matcha}
 or network throughput \citep{marfoq2020throughput}. 
Unlike our approach, these methods simply seek to optimize the connectivity of
 the topology while respecting some communication constraints, but they do not take into account the data distributions across nodes. 

\looseness=-1 Until recently, \cite{bellet2021d} was the only approach
that leverages the distribution of
data in the design of the topology. Focusing on classification under label
 skew, they propose a heuristic approach that consists of inter-connected
 cliques, where class proportions in each clique should
 be as close as possible to the global proportions. 
Our approach is more flexible: it can learn more general topologies, and
provides full control over their sparsity. Furthermore, our
topology learning criteria is theoretically justified, while the one in \cite{bellet2021d} is only supported by empirical experiments.
We think however that the ideas of the present paper could pave the way for a
theoretical analysis of their work.

Concurrent to and independently from our work, \cite{dandi2022data} provide a
similar analysis of the convergence rate of D-SGD using a quantity called
``relative heterogeneity''. However, our approaches differ greatly in how they learn the topology. In fact, \cite{dandi2022data} do not learn the topology itself (i.e., which nodes are connected) but only the weights of a predefined topology. In other words, the set of edges is fixed in advance. This severely limits the ability to mitigate the effect of data heterogeneity unless the predefined topology is dense. In contrast, our approach learns a sparse topology (both the edges and their associated mixing weights) in order to balance the convergence rate and the communication complexity of D-SGD.

\section{Preliminaries}
\label{sec:background}

\textbf{Problem setting.}
\looseness = -1 In decentralized federated learning, $n\in \mathbb{N}^\star$ agents (nodes) with their own data distribution seek to collaborate in order to solve a global consensus problem. Formally, the agents aim to learn a single parameter $\theta \in \RR^d$ so as to optimize the global objective \citep{lian2017can}: 
\begin{equation}
    \label{eq:obj}
    \textstyle
    f^* \triangleq \min_{\theta \in \RR^d}\big[f(\theta)\triangleq \frac{1}
    {n}\sum_{i=1}^n f_i(\theta)\big],
\end{equation}
where $f_i(\theta) \triangleq \EEE{Z_i \sim \calD_i}[F_i(\theta;Z_i)]$ is the local objective function associated to node $i$. The random vector $Z_i$ is drawn from the data distribution $\calD_i$ of agent $i$, having support over a space $\Omega_i$, and $F_i:\RR^d\times \Omega_i\rightarrow \RR$ is its \emph{pointwise} loss function (differentiable in its first argument). %
Note that the distributions $\calD_i$ can be very different, which is
common in real applications 
\citep{kairouz2021advances}. From an optimization point of view, this means that a local optimum 
$\theta_i^\star\in \arg\min_{\theta}f_i(\theta)$ can be far from a global optimum $\theta^\star$ of (\ref{eq:obj}).

\looseness =-1 To collaboratively solve (\ref{eq:obj}) in a fully decentralized manner, the
agents communicate with each other over a directed graph. The graph topology
is represented by a matrix $W\in[0,1]^{n\times n}$, where $W_{ij}>0$ gives the
weight that agent $i$ gives to messages received from agent $j$, while
$W_{ij}=0$ (no edge) means that $i$ does not receive messages from $j$.
The choice of topology $W$ affects the trade-off between the convergence rate of decentralized optimization algorithms and the communication costs. Indeed, more edges imply higher communication costs but often faster convergence.
Communication costs, or \emph{per-iteration complexity}, are often regarded as proportional to the maximum (in or out)-degrees of nodes in the topology, representing the maximum (incoming or outcoming) load of a node \citep{lian2017can}:
\begin{equation}
\begin{aligned}
    d^{\text{in}}_{\text{max}}(W) & =\textstyle \max_i \sum_{j=1}^n \mathbb{I}[W_{ij} > 0], \\
d^{\text{out}}_{\text{max}}(W) & =\textstyle \max_i \sum_{j=1}^n \mathbb{I}[W_{ji} > 0].
\end{aligned}
\label{eq:degrees}
\end{equation}
From this perspective, the complete graph, and the star topology induced by server-based federated learning, yield high communication costs, as the maximum degree is $n-1$.

\textbf{Decentralized SGD.}
Decentralized Stochastic Gradient Descent (D-SGD) \citep{lian2017can,koloskova20} is a popular fully decentralized algorithm for solving problems of the form (\ref{eq:obj}). As mentioned above, such algorithms operate on a graph topology represented by the matrix $W\in [0,1]^{n\times n}$. In particular, D-SGD requires that $W$ is a \emph{mixing} matrix, i.e. doubly stochastic: $W\1 = \1$ and $\1^\transpose W = \1^\transpose$. \looseness = -1


      \begin{algorithm}[t]
        \caption{Decentralized SGD \citep{lian2017can}}\label{alg:d-sgd} 
        \begin{algorithmic}
            \Require Initialize $\forall i$, $\theta_i^{(0)} = \theta^{(0)} \in \RR^d$, iterations $T$, stepsizes $\{\eta_t\}_{t=0}^{T-1}$, mixing $\{W^{(t)}\}_{t=0}^{T-1}$.
            \For{$t=0,\ldots,T-1$}
            \For{each node $i=1,\ldots,n$ (in parallel)}
                \State Sample $Z_i^{(t)}\sim \calD_i$
                \State $\theta_i^{(t + \frac{1}{2})} \gets \theta_i^{(t)} - \eta_t\nabla F_i(\theta^{(t)}_i,Z^{(t)}_i)$
                \State $\theta_i^{(t+1)} \gets \sum^n_{j=1}W^{(t)}_{ij}\theta_j^{(t + \frac{1}{2})}$
            \EndFor
            \EndFor            
        \end{algorithmic}
      \end{algorithm}
In the rest of the paper, we will use the terms topology and mixing matrix interchangeably. 
For sake of generality, we consider a setting where the mixing matrix may change at each iteration \citep{koloskova20}.
On the other hand, we assume for simplicity that the mixing matrices are deterministic. 
All our results can however be extended to random
mixing matrices, see Appendix~\ref{sec:random-mixing} for details.

D-SGD is summarized in Algorithm \ref{alg:d-sgd}. At iteration $t$, each node
 $i$ first updates its local estimate $\theta_i^{(t)}$ 
 based on $\nabla F_i(\theta^{(t)}_i,Z^{(t)}_i)$, the stochastic gradient
 of $F_i$ evaluated at $\theta^{(t)}_i$ with $Z^{(t)}_i$ sampled from $\calD_i$. Then, each node aggregates its current parameter value with its neighbors according to the mixing matrix $W^{(t)}$. 

\textbf{General assumptions.}
We recall some standard assumptions extensively considered in decentralized
learning
\citep{bubeck2014convex, nguyen2019new,lian2017can,tang2018d,assran2019stochastic,li2019communication,kong2021consensus,ying2021exponential}.

\begin{assumption} \emph{($L$-smoothness)}
    \label{ass:smoothness}
    There exists a constant $L>0$ such that for any $Z \in \Omega_i$, $\theta,
    \tilde{\theta}\in~\RR^d$ we have $\| \nabla F_i(\theta,Z) - \nabla F_i(\tilde{\theta},Z)\| \leq L \|\theta- \tilde{\theta}\|~$.
\end{assumption}

\begin{assumption} \emph{(Bounded variance)}
    \label{ass:variance}
    For any node $i \in \llbracket 1,\ldots, n\rrbracket$, there exists a constant $\sigma_i^2>0$ such that for any $\theta \in \RR^d$, we have $\EEE{Z \sim \calD_i}\big[\left\|\nabla F_i(\theta,Z) - \nabla f_i(\theta)\right\|_2^2\big] \leq \sigma_i^2~$.
\end{assumption}


\begin{assumption}\emph{(Mixing parameter)}
    \label{ass:rho}
    There exists a mixing parameter $p\in [0,1]$ such that for any matrix $M \in \RR^{d \times n}$, we have $\|MW^\transpose-\overbar{M}\|_F^2\leq (1-p)\|M-\overbar{M}\|_F^2~$, 
where $\|\cdot\|_F$ denotes the Frobenius norm and $\overbar{M} = M(\frac{1}{n}\1\1^\transpose)$. 
\end{assumption}

Assumption~\ref{ass:rho} measures how well an averaging step using a mixing matrix $W$ brings an arbitrary matrix $M$ closer to $\overbar{M}$. 
It is always verified for $p=1-\lambda_2(W^\transpose W)$ with $\lambda_2(W^\transpose W)$ the second largest eigenvalue of $W^\transpose W$ \citep{boyd2006randomized} \looseness =-1 . 

\section{Joint Effect of Topology and Data Heterogeneity}
\label{sec:topo-vs-heterogeneity}

In this section, we introduce a new quantity, called \emph{neighborhood
heterogeneity}, and derive new convergence rates for D-SGD that depend on this quantity.
These rates have several nice properties: (i) they hold under weaker assumptions than previous work (unbounded local heterogeneity), (ii) they highlight the interplay between the topology and the heterogeneous data distribution across nodes, and (iii) they provide a criterion for choosing topologies not only based on their mixing properties but also based on data. \looseness=-1
\subsection{Neighborhood Heterogeneity}
\label{sec:our_quantity}

Given a mixing matrix $W$, our notion of neighborhood heterogeneity measures
the expected distance between the \emph{aggregated gradients in the neighborhood of a node} (as weighted by $W$) and the \emph{global average of gradients}. In our analysis, we will assume this distance to be bounded.

\begin{assumption}[Bounded neighborhood heterogeneity] 
        \label{ass:new}
        There exists a constant $\bar{\tau}^2>0$ such that  $\forall \theta \in \RR^d$: 
        \begin{equation}
            \label{eq:new}
            \frac{1}{n}\sum_{i=1}^n\EE\,\Big\|\sum_{j=1}^nW_{ij}\nabla
F_j(\theta,Z_j) - \frac{1}{n}\sum_{j=1}^n\nabla F_j
(\theta,Z_j)\Big\|_2^2 \leq \bar{\tau}^2.
        \end{equation}
    \end{assumption}

To better understand Assumption \ref{ass:new}, we can upper-bound the left-hand term of the previous equation, denoted $H(\theta)$, using a bias-variance decomposition. This leads to the following bound:
\begin{equation}
\begin{aligned}
    H(\theta) 
    \leq \frac{1}{n}\sum_{i=1}^n\Big\|\sum_{j=1}^n & W_{ij}\nabla f_j(\theta) - \nabla f(\theta)\Big\|_2^2 \\
    & + \frac{\sigma_{\max}^2}{n}\|W-\frac{1}{n}\1\1^\transpose\|_F^2~, \label{eq:H_bias_variance}
\end{aligned}
\end{equation}
with $\sigma_{\max}^2=\max_i\sigma_i^2$. This upper bound contains two terms. The first one is a \emph{bias term},
 related to the heterogeneity of the problem. It essentially measures how the
 gradients of local objectives
 differ from the
 gradient of the global objective when they are aggregated at the
 neighborhood level of the topology through $W$. The second one is a \emph{variance term} closely related 
to the mixing parameter $p$ of Assumption \ref{ass:rho}: we can show that it is
upper bounded by $\sigma_{\max}^2(1-p)$ and
lower bounded by $\sigma_{\max}^2(1-p)/n$, see
Proposition~\ref{prop:relation-p-frob} in Appendix \ref{app:other-proofs}.%
%



\textbf{Comparison to classic bounded heterogeneity assumption.}
In our analysis, we use Assumption~\ref{ass:new} in replacement of the 
\emph{bounded local heterogeneity} condition used in previous literature 
\citep{lian2017can,Lian2018,assran2019stochastic,koloskova20,ying2021exponential}. We recall it below.

\begin{assumption}[Bounded local heterogeneity]\label{ass:heterogeneityA}
There exists a constant $\bar{\zeta}^2>0$ such that $\frac{1}{n}\sum_{i=1}^n\|\nabla f_i(\theta) - \nabla f(\theta)\|_2^2\leq \bar{\zeta}^2,$ $\forall \theta \in \RR^d$.
\end{assumption}

Assumption~\ref{ass:heterogeneityA} has the same form as the bias term of in Equation \eqref{eq:H_bias_variance} but considers $W=I$ (i.e., it does not depend on the
topology). It requires that the local
gradients should not be too far from
the global gradient:
the more heterogeneous the nodes' distribution (and objectives), the bigger $\bar{\zeta}^2$.
In contrast, neighborhood heterogeneity takes into account the mixture of
gradients in the neighborhoods defined by $W$.
Crucially, Assumption~\ref{ass:new} is more flexible than
Assumption~\ref{ass:heterogeneityA}. More precisely, our set of assumptions (Assumptions~\ref{ass:variance}-\ref{ass:new}) is less restrictive than those in
previous work (Assumptions~\ref{ass:variance}, \ref{ass:rho}, 
\ref{ass:heterogeneityA}). To see this, we first show that our set of
assumptions is implied by the latter (proof in Appendix~\ref{app:other-proofs}).


\begin{proposition}
    \label{prop:H_bound}
    Let Assumptions \ref{ass:variance}-\ref{ass:rho} and \ref{ass:heterogeneityA} to be verified. Then Assumption~\ref{ass:new} is satisfied with
    $\bar{\tau}^2 = (1-p)\left(\bar{\zeta}^2 + \bar{\sigma}^2\right)$,
    where $\bar{\sigma}^2\triangleq\frac{1}{n}\sum_i\sigma_i^2$.
\end{proposition}


We now show that our set of assumptions (\ref{ass:variance}-\ref{ass:new}) is strictly more general than Assumptions~\ref{ass:variance}, \ref{ass:rho}, 
\ref{ass:heterogeneityA} by identifying situations where Assumption ~\ref{ass:new} is verified while Assumption~\ref{ass:heterogeneityA} is not. 
A trivial example is the
complete graph $W=\frac{1}{n}\1\1^\transpose$, for which we have $
\bar{\tau}^2 =0$, regardless of heterogeneity. 
More interestingly, some combinations of \emph{sparse} topologies and
 data distributions can ensure that $\bar{\tau}^2$ remains small while
$\bar{\zeta}^2$ can be arbitrary large. We give a simple example below (detailed derivations in Appendix \ref{app:examples}). \looseness = -1 


\begin{example}[Two clusters and a ring topology]
\label{ex:ring}
Let $n$ be an even number and assume $Z_i\sim \calD_i \triangleq \mathcal{N}
(m,\tilde{\sigma}^2)$ if $i$ is odd and $Z_i\sim  \calD_i \triangleq 
\mathcal{N}(-m,\tilde{\sigma}^2)$ if $i$ is even. Let $
\tilde{\sigma}^2<+\infty$ (necessary to have Assumption~\ref{ass:variance})
and $m>0$ potentially asymptotically large. We fix $F_i
(\theta,Z_i) = (\theta - Z_i)^2$ (mean estimation). Consider a ring topology
that alternates between one odd node and one even node, with the diagonal and
off-diagonal entries of $W$ equal to $1/2$ and $1/4$ respectively. Then we
have
$\bar{\tau}^2 = \sigma^2_i = 4\tilde{\sigma}^2 < +\infty$, while $\bar{\zeta}^2=4m^2$ can be arbitrarily large as $m$ grows.
\end{example}


This illustrates that an appropriate topology, even as sparse as a ring, can
control $\bar{\tau}^2$ and mitigate the underlying heterogeneity of the problem.
In Section~\ref{sec:graph_learn}, we will show that we can learn a sparse
topology $W$ that (approximately) minimizes the neighborhood heterogeneity bound $\bar{\tau}^2$. Before that, we validate the relevance of our new
Assumption~\ref{ass:new} by deriving a novel convergence result for D-SGD.

\subsection{Convergence Analysis}

We now present the main theoretical result of this section: two new non-asymptotic convergence results for D-SGD under Assumption~\ref{ass:new}. 
The proof of this theorem is given in Appendix~\ref{app:proof_thme}.

\begin{theorem}
    \label{thme:main}
    Consider Algorithm~\ref{alg:d-sgd} with mixing matrices $W^{(0)},\ldots,W^{(T-1)}$ satisfying Assumptions~\ref{ass:rho} and \ref{ass:new}. Assume further that Assumptions \ref{ass:smoothness}-\ref{ass:variance} are respected, and denote $\bar{\theta}^{(t)} \triangleq \frac{1}{n}\sum_{i=1}^n \theta_i^{(t)}$. For any target accuracy $\varepsilon>0$, there exists a constant stepsize $\eta\leq \eta_{\max}=\frac{p}{8L}$ such that:

    \begin{minipage}{0.45\textwidth}    
    \textbf{Convex case:} \\

    \vspace*{-5pt}
    $\frac{1}{T+1}\sum_{t=0}^T\EE (f(\bar{\theta}^{(t)}) - f^\star) \leq \varepsilon$ as soon as 
    \begin{equation}
        \label{eq:Orate}
       T \geq \calO\Big(\frac{\bar{\sigma}^2}{n \varepsilon^2} + \frac{\sqrt{L}\bar{\tau}}{p \varepsilon^\frac{3}{2}} + \frac{L}{p\varepsilon}\Big)r_0~,
    \end{equation} 
    \end{minipage}%
     \hfill
     \noindent
     \begin{minipage}{0.45\textwidth}    
    \textbf{Non-convex case:} \\

    \vspace*{-5pt}
    $\frac{1}{T+1}\sum_{t=0}^T\EE \|\nabla f(\bar{\theta}^{(t)}\|_2^2 \leq \varepsilon$ as soon as 
    \begin{equation}
        \label{eq:OrateNC}
       T \geq \calO\Big(\frac{L\bar{\sigma}^2}{n \varepsilon^2} + \frac{L\bar{\tau}}{p \varepsilon^\frac{3}{2}} + \frac{L}{p\varepsilon}\Big)f_0~,
    \end{equation}
\end{minipage}

where $T$ is the number of iterations, $r_0=\|\theta^{(0)}-\theta^\star\|_2^2$, %
$f_0 = f(\theta^{(0)}) - f^\star$ and $\calO(\cdot)$ hides the numerical constants explicitly provided in the proof.
\end{theorem}

\textbf{Analysis and comparison to prior results.}
To put the above theorem into perspective, recall that Centralized (Parallel)
Stochastic Gradient Descent 
(C-PSGD) %
is equivalent to D-SGD with the mixing matrix $W=\frac{1}{n}\1\1^\transpose$ (complete graph). For this specific case, it has been shown that in the convex scenario, an accuracy $\varepsilon$ is achieved after
$T\geq \calO(\frac{\bar{\sigma}^2}{n\varepsilon^2}+\frac{L}
{\varepsilon})$ iterations \citep{dekel2012optimal,bottou2018optimization, stich2020error}.
On the other hand, existing results for D-SGD (under
Assumption~\ref{ass:heterogeneityA} instead of Assumption~\ref{ass:new}) require
$T\geq \calO(\frac{\bar{\sigma}^2}{n\varepsilon^2}+\frac{\sqrt{L(1-p)}(\bar{\zeta}+\bar{\sigma}\sqrt{p})}{p\varepsilon^{3/2}}+\frac{L}{p\varepsilon})$ iterations \citep{koloskova20}.
 %
%

The first thing to note is that rate \eqref{eq:Orate} is consistent with the
above rates. When the complete graph topology $W=\frac{1}
{n}\1\1^\transpose$ is used at each iteration we have $\bar{\tau}=0$ and
 $p=1$, which allows us to recover the rate of the communication-inefficient C-PSGD. Furthermore,
 considering the classical Assumption~\ref{ass:heterogeneityA} and using
 Proposition \ref{prop:H_bound} gives the looser bound $\calO(\frac{
\bar{\sigma}^2}{n\varepsilon^2}+\frac{\sqrt{L(1-p)}(\bar{\zeta}+\bar{\sigma})}
{p\varepsilon^{3/2}}+\frac{L}{p\varepsilon})$ which is equivalent to the rate of D-SGD in \cite{koloskova20}. Similarly, the rate \eqref{eq:OrateNC} obtained for non-convex objectives is also consistent with  \cite{koloskova20}. %

Crucially, recall that in the heterogeneous setting $\bar{\tau}$ can be much smaller than $\sqrt{1-p}(\bar{\zeta}+\bar{\sigma})$ (see Section \ref{sec:our_quantity}), which makes our bounds sharper.
This is because the topology now influences the convergence rate in
 Theorem \ref{thme:main} via both the mixing parameter $p$ and $\bar{\tau}$. This is of particular significance in situations where
 communication constraints are strong so that the topology connectivity has to
 be low (i.e., $p$ close to $0$). In that case, prior rates are heavily
 impacted by data heterogeneity as $p$ can no longer compensate for it. In
 contrast, we can expect that a well-chosen sparse topology can achieve small $\bar{\tau}$ and thus mitigate the impact of data heterogeneity. To highlight this, we can go back to Example~\ref{ex:ring}. For the chosen ring topology, we have $p=\Theta(\frac{1}{n^2})$, but the specific arrangement of nodes and the weights in $W$ still allow a small bound $\bar{\tau}^2$ on neighborhood heterogeneity. \looseness = -1%

%

\section{Learning the Topology}
\label{sec:graph_learn}

In the previous rates (\ref{eq:Orate}) and \eqref{eq:OrateNC}, the smaller the bound $\bar{\tau}^2$ on neighborhood heterogeneity, the fewer iterations needed to reach an error $\varepsilon$. %
This motivates the idea of learning a \emph{sparse} topology $W$
that \emph{approximately} minimizes neighborhood heterogeneity (Equation~\eqref{eq:new}), in order to control the trade-off between
the convergence rate and the per-iteration communication
complexity given in Equation~\eqref{eq:degrees}.
However, minimizing neighborhood heterogeneity in the general setting appears
to be challenging without further statistical assumptions, as Equation~\eqref{eq:new} should hold for all $\theta\in\mathbb{R}^d$.
Below, we focus on \emph{classification with label skew}, and show that Equation~\eqref{eq:new} simplifies to a more tractable quantity. \looseness=-1
\subsection{Statistical Learning with Label Skew}
\label{sec:framework_label_skew}

Label skew is an important type of data heterogeneity in federated
classification problems \citep{kairouz2021advances,quagmire,bellet2021d}.
In this setting, each agent $i$ is associated with a random variable $Z_i = (X_i,Y_i)\sim \calD_i$ where $X_i\in\mathbb{R}^q$ represents the feature vector and $Y_i\in \llbracket 1,\ldots, K\rrbracket$ the associated class label. The agents aim to learn a classifier $h_\theta:\mathbb{R}^q\rightarrow \llbracket 1,\ldots, K\rrbracket$ parameterized by $\theta\in\mathbb{R}^p$ such that $h_\theta(X_i)$ is a good predictor of $Y_i$ for all $i$.
The heterogeneity of the distributions $\{\calD_i\}_{i=1}^n$ comes only from a
\emph{difference in the label distribution} $P_i(Y)$ i.e. $\calD_i = P_i(X,Y)=P(X|Y)P_i(Y)$.
 For simplicity, we assume that all agents use the same pointwise loss
 function ($F_i=F$ for all $i$), which is typically the cross-entropy.

Under the above framework, we can derive
a neighborhood heterogeneity bound $\bar{\tau}^2$ that can effectively be minimized with respect to $W$. 

\begin{proposition}[Bounded neighborhood heterogeneity under
label skew]\label{prop:upper-bound-H}
    Consider the statistical framework defined above and assume there exists $B>0$ such that
    $\forall k=1,\ldots,K$ and $\forall\theta \in\RR^d$, $\|\EE_X[\nabla F(\theta ; X,Y)|Y=k] - \frac{1}{K}\sum_{k^\prime = 1}^K \EE_X[\nabla F(\theta ; X,Y)|Y=k^\prime]\|_2^2\leq B$.
 Then, denoting $\pi_{jk}\triangleq P_j(Y=k)$, Assumption \ref{ass:new} is satisfied with:
 \begin{equation}
    \begin{aligned}
       \bar{\tau}^2 =  \frac{KB}{n}\sum_{k=1}^K\sum_{i=1}^n  \Big(\sum_{j=1}^n  & W_
        {ij}\pi_{jk}  - \frac{1}{n}\sum_{j=1}^n \pi_{jk}\Big)^2 \\
        & + \frac{\sigma_{\max}^2}{n}\|W-\frac{1}{n}\1\1^\transpose\|_F^2~.
    \end{aligned}
    \label{eq:prop2}
\end{equation}
\end{proposition}

The proof is provided in Appendix~\ref{app:other-proofs}.
Note that the condition involving $B$ corresponds to a bounded heterogeneity
 assumption at the class level (rather than at the agent level as in
 Assumption~\ref{ass:heterogeneityA}). 

The neighborhood heterogeneity bound $\bar{\tau}^2$ in \eqref{eq:prop2} is
quadratic in $W$ and composed of two terms.
The first one is a \emph{bias} term due to the label skew: it will be minimal
 if neighborhood-level class proportions (weighted by $W$) match the global
 class proportions. This is trivially achieved for any choice of $W$ if the
 class proportions are the same across nodes. 
The second term is a \emph{variance} term which is minimal when $W=
\frac{\1\1^\transpose}{n}$, the complete topology with uniform weights. As a matter of fact, this topology is also the unique global minimizer of \eqref{eq:prop2}, which is equal to $0$ in this case.
However, as already discussed, such a dense mixing matrix is impractical as it yields huge communication costs. %
We will show how the per-iteration communication complexity of D-SGD can be controlled
while \emph{approximately} minimizing $\bar{\tau}^2$ in \eqref{eq:prop2}.

\subsection{Optimization with the Frank-Wolfe Algorithm}

In this section, we design an algorithm that finds a sparse approximate minimizer of $\bar{\tau}^2$ in \eqref{eq:prop2}. We focus on learning a single mixing matrix $W$ %
as a ``pre-processing'' step (i.e., before running D-SGD), and do so in a
centralized manner. Specifically, we assume that a single party (which may be one of the agents, or a third-party) has access to the class proportions $\pi_{ik} = P_i(Y=k)$ for each agent $i$ and each class $k$.
In practice, since each agent has access to its local dataset, it can compute these local proportions locally and share them without sharing the local data itself.


\textbf{Optimization problem.} Our objective is to learn a \emph{sparse}
mixing matrix $W$
which \emph{approximately} minimizes
$\bar{\tau}^2$ in \eqref{eq:prop2}. Denoting by
$\calS \triangleq \left\{W \in [0,1]^{n \times n} : W\1 = \1 , \hspace{0.5em} \1^\transpose W = \1^\transpose \right\}$ the set of doubly stochastic matrices, the optimization problem can be written as follows:
\begin{equation}
    \label{eq:obj_FW}
    \textstyle
    \underset{W\in \mathcal{S}}{\min}~\Big\{ g(W)\triangleq  \frac{1}{n}\Big\|W \Pi -  \frac{\1\1^\transpose}{n}\Pi\Big\|_F^2 + \frac{\lambda}{n} \Big\|W -  \frac{\1\1^\transpose}{n}\Big\|_F^2\Big\}~,
\end{equation}
where $\Pi\in[0,1]^{n\times K}$ contains the class proportions $\{\pi_{ik}\}$
and $\lambda>0$ is a hyperparameter. To exactly match \eqref{eq:prop2},
$\lambda$ should be equal to $\frac{\sigma_{\max}^2}{KB}$, but $\sigma_
{\max}^2$ and $B$ are unknown in
practice. Instead, we use $\lambda$ to control the bias-variance
trade-off.
As discussed in Section~\ref{sec:our_quantity}, the variance term is an upper
 bound of $1-p$ with $p$ the mixing parameter of $W$. 
Therefore, $\lambda$ allows to tune a trade-off between the minimization of the bias due to label skew and the maximization of the mixing parameter of $W$.

\textbf{Algorithm.}
%
We propose to find sparse approximations of \eqref{eq:obj_FW} using a
Frank-Wolfe (FW) algorithm, which is well-suited to
learn a sparse parameter over convex hulls of finite set of atoms \citep{jaggi2013revisiting}. In our case,
$\calS$ corresponds to the convex hull of the set $\mathcal{A}$ of all permutation matrices \citep{lovasz2009matching,tewari2011greedy,valls2020birkhoff}.

\begin{algorithm}[t]
    \caption{Sparse Topology Learning with Frank-Wolfe (STL-FW) \looseness =-1}\label{alg:FW-W} 
    \begin{algorithmic}
    \Require Initialization $\widehat{W}^{(0)}=I_n$, class proportions $\Pi\in[0,1]^{n\times K}$ and hyperparameter $\lambda>0$.
    \For{$l = 0,\ldots,L$}
        \State $P^{(l+1)} = {\arg \min}_{P \in \mathcal{A}}~ \langle P , \nabla g(\widehat{W}^{(l)}) \rangle$ 
        \State $\gamma^{(l+1)} =  {\arg \min}_{\gamma \in [0,1]}~ g\big((1-\gamma) \widehat{W}^{(l)} + \gamma P^{(l+1)}\big)$ 
        \State $\widehat{W}^{(l+1)} = (1-\gamma^{(l+1)}) \widehat{W}^{(l)} + \gamma^{(l+1)} P^{(l+1)}$ 
    \EndFor
    \end{algorithmic}
    \end{algorithm}

    \newcommand{\ratio}{0.31}
    \newcommand{\ratiob}{0.70}
    \begin{figure*}[t]
        \centering
        \begin{minipage}[t]{\ratio\linewidth}
            \centering
            \includegraphics[height=\ratiob\linewidth]{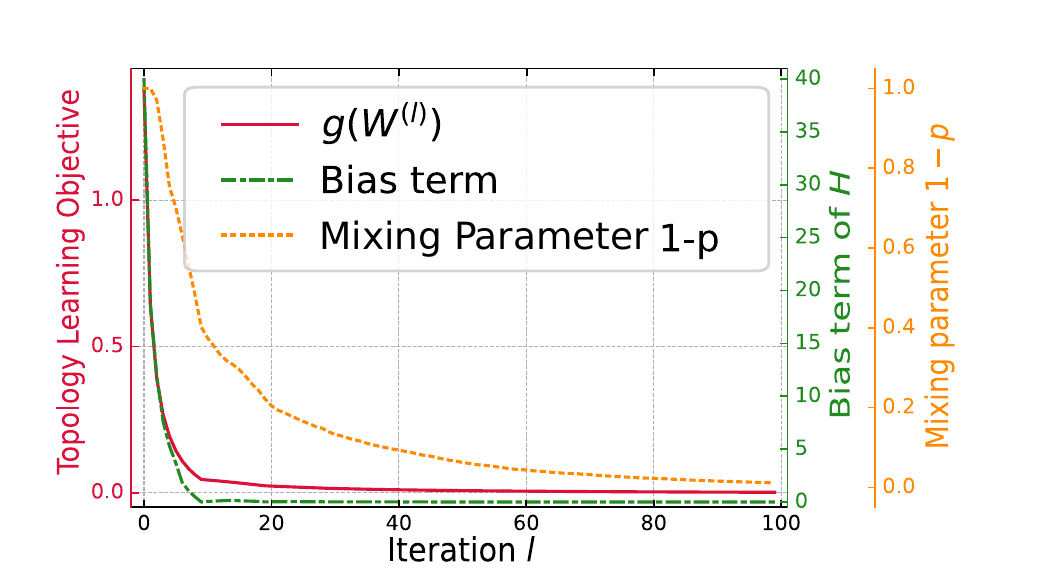}
            {\small (a) Topology learning\;}
        \end{minipage}\hfill
        \begin{minipage}[t]{\ratio\linewidth}
            \centering
            \includegraphics[height=\ratiob\linewidth]{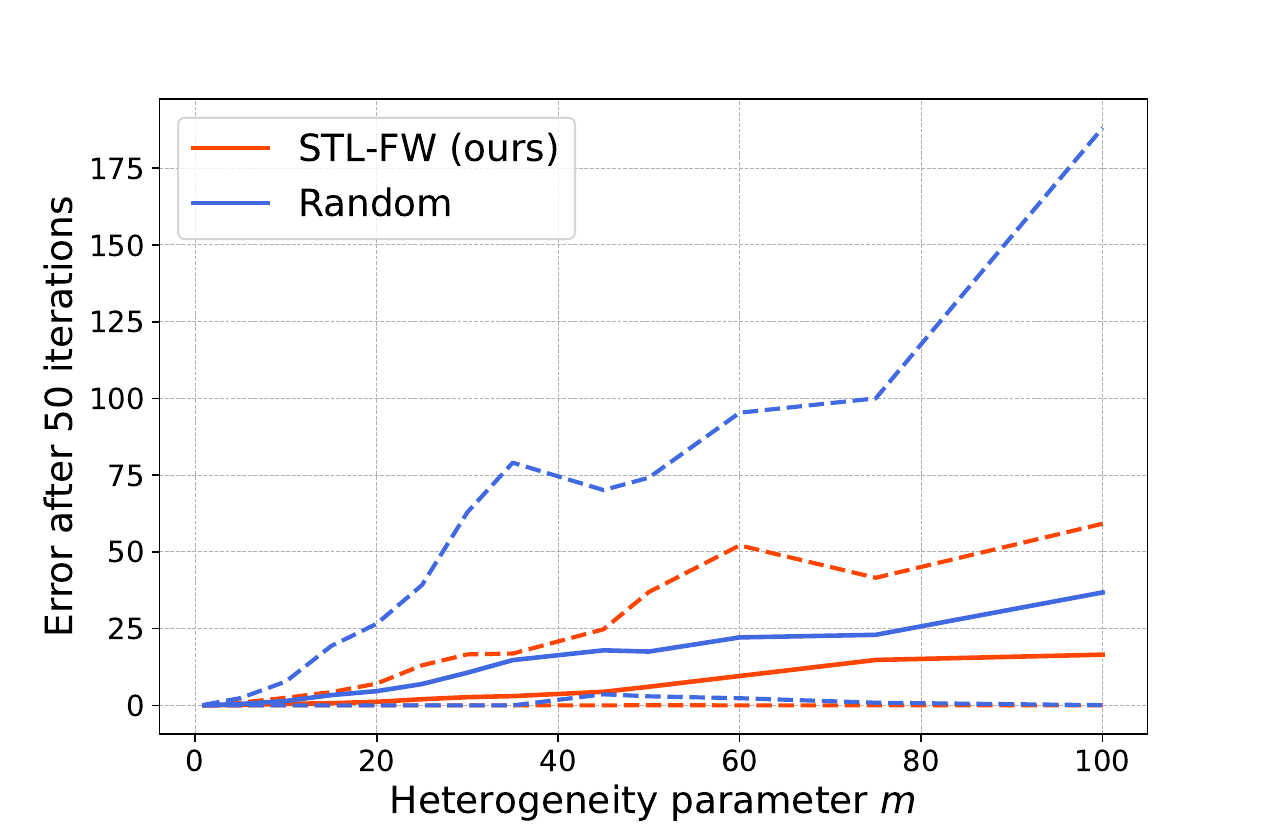}
            {\small (b)\; D-SGD with $d_{\text{max}} = 3$ }
        \end{minipage}
        \begin{minipage}[t]{\ratio\linewidth}
            \centering
            \includegraphics[height=\ratiob\linewidth]{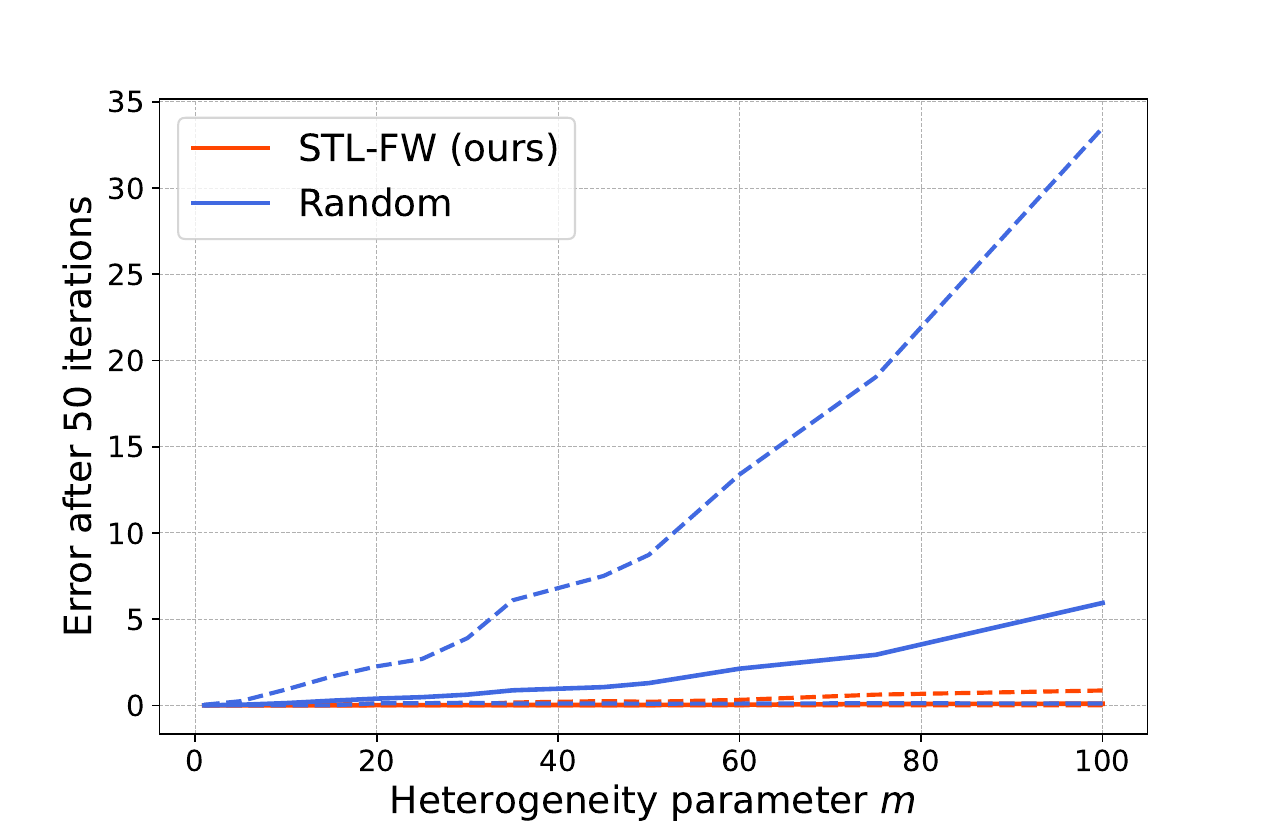}
            {\small (c)\; D-SGD with $d_{\text{max}}=9$}
        \end{minipage}
        \caption{\looseness=-1 \textbf{(a)} Evolution of key quantities across the
        iterations of topology learning: in red the objective function $g(W^{
        (l)})$, in green the
        bias term $\sup_\theta \frac{1}{n}\sum_{i=1}^n\|\sum_{j=1}^nW^{(l)}_
        {ij}\nabla f_j
        (\theta) - \nabla f(\theta)\|_2^2$
        and in yellow the mixing
        parameter $1-p = \lambda_2(W^{(l)\transpose }W^{(l)})$. Here, $\lambda =
        0.5$ and $m=5$. \textbf{(b, c)} Error $n^{-1}\|\theta^{(t)} -
        \theta^\star\|_2^2$ (solid line) of D-SGD after $50$ iterations,
        averaged over
        $10$ runs, for increasing levels of heterogeneity (measured by parameter
        $m$). The dashed lines show
        $\max_i(\theta_i^{(t)} - \theta^\star)^2$ and $\min_i(\theta_i^{(t)} -
        \theta^\star)^2$, illustrating the variability across
        nodes.}
        \label{fig:experiments_synth}
    \end{figure*}

The algorithm is summarized in Algorithm~\ref{alg:FW-W}. Starting from the identity matrix 
$\widehat{W}^{(0)}=I_n\in \calS$, each iteration $l\geq 0$ consists of moving towards
a feasible point $P^{(l+1)}$ that minimizes a linearization of $g$
at the current iterate $\widehat{W}^{(l)}$. As finding $P^{(l+1)}$ is a linear
problem, solving it over $\calS$ is equivalent to solving it over $
\mathcal{A}$. Although $\mathcal{A}$ contains $n!$ elements, the linear
program corresponds to the well-known \emph{assignment problem} 
\citep{burkard2012assignment, crouse2016implementing} and can be solved
tractably with the Hungarian algorithm, which has a worst-case complexity of
$\calO(n^3)$~\citep{lovasz2009matching}.\footnote{The algorithm
is quite fast in practice: for instance, the scipy implementation runs in 0.3s
on a regular laptop for $n=1000$.}
Note that the gradient $\nabla g(W)$ needed to solve the assignment problem is given by \looseness = -1%
\begin{equation*}
    \textstyle
    \frac{2}{n}\sum_{k=1}^K(W\Pi_{:,k}-\overline{\Pi_{:,k}}\1)\cdot \Pi_{:,k}^\transpose + \frac{2}{n}\lambda \left(W - \frac{\1\1^\transpose}{n}\right)~,
\end{equation*}
\looseness=-1 where $\Pi_{:,k}$ is the $k$-th column of $\Pi$.
The next iterate $\widehat{W}^{(l+1)}$ is then obtained as a convex combination of $P^{(l+1)}$ and $\widehat{W}^{(l)}$, and is thus guaranteed to be in $\calS$. The optimal combining weight is computed by line-search, which has a closed-form solution since $g$ is quadratic (see Appendix~\ref{sec:line-search}).

\looseness=-1 Crucially, Algorithm~\ref{alg:FW-W} allows to control the
sparsity of the
final solution: since a permutation matrix contains exactly one non-zero entry
in each row and each column, at most one new incoming and one new outgoing
edge per node are added. As we start from the identity matrix (i.e., only
self-edges), this guarantees that at the end of the $l$-th iteration, each
node will have at most $l$ in-neighbors and $l$ out-neighbors. The
per-iteration communication complexity of D-SGD induced by the learned topology can
thus be directly controlled by the number of iterations of our algorithm. The
trade-off with the quality of the solution is
quantified by the following theorem, which is derived from standard results for FW \citep{jaggi2013revisiting} combined with a tight bound on the smoothness of $g$ in appropriate norm (see Appendix~\ref{app:other-proofs}).
\begin{theorem} \label{coro:bound}
    Consider the statistical setup presented in Section~\ref{sec:framework_label_skew} and let $\{\widehat{W}^{(l)}\}^L_{l = 1}$
    be the sequence of mixing matrices generated by Algorithm~\ref{alg:FW-W}. Then, at any iteration $l=1,\ldots,L$, we have:
    \begin{equation}
        \label{eq:good_bound_g}
        g(\widehat{W}^{(l)}) \leq \frac{16}{l+2}\big(\lambda + \frac{1}
        {n}\big\|\sum_{k=1}^K(\Pi_{:,k}-\overline{\Pi_{:,k}}\1)\cdot \Pi_
        {:,k}^\transpose \big\|^\star_{2} \big)~,
    \end{equation}
    where $\left\|\cdot\right\|^\star_{2}$ stands for the nuclear norm, i.e., the sum of singular values.
Furthermore, we have $d^{\text{in}}_{\text{max}}(\widehat{W}^{(l)})\leq l$ and $d^{\text{out}}_{\text{max}}(\widehat{W}^{(l)})\leq l$,
resulting in a per-iteration complexity bounded by $l$.
\end{theorem}

The above theorem shows that the objective $g$ decreases at a rate of $\calO(1/l)$ as new connections between nodes are made. 
In general, we can bound \eqref{eq:good_bound_g} less tightly by  $g(\widehat{W}^{(l)}) \leq \frac{16}{l+2}\left(\lambda + 1 \right)$, which is \emph{independent of the number of nodes} $n$.
Recall that with $\lambda = \sigma_{\max}^2/KB$, the value $
\bar{\tau}^2$ of Proposition \ref{prop:upper-bound-H} is exactly equal to
$KB\cdot g(W)$. Therefore, the
 bound given in Theorem~\ref{coro:bound} directly bounds neighborhood
 heterogeneity and can thus be plugged in the rates of Theorem 
\ref{thme:main}. 

To summarize, our approach provides a principled way to learn the
topology
so as to reduce neighborhood heterogeneity while
controlling the per-iteration communication complexity of D-SGD. Remarkably, the fact that $g(\widehat{W}^{(l)})$ is independent of $n$ implies that we can find topologies that approximately optimize the convergence rate of D-SGD while keeping the communication load per node constant, thereby guaranteeing scalability to a large number of nodes even in highly heterogeneous scenarios.

\section{Experiments}
\label{sec:xps}

This section shows the practical usefulness of our topology learning method, referred to as Sparse Topology Learning with Frank-Wolfe (STL-FW). We call \emph{communication budget} $d_{\text{max}} = \max\{d^{\text{in}}_{\text{max}}, d^{\text{out}}_{\text{max}}\}$ the maximal number of neighbors a node can have in the used topologies, which controls the per-iteration communication complexity incurred by any node. \looseness = -1 
\subsection{Simulations on Synthetic Data}
\label{sec:synthetic}

\newcommand{\ratiorwexp}{0.33}
\begin{figure*}[t]
    \centering
    \begin{minipage}[t]{\ratiorwexp\linewidth}
        \centering
        \includegraphics[height=\ratiob\linewidth]{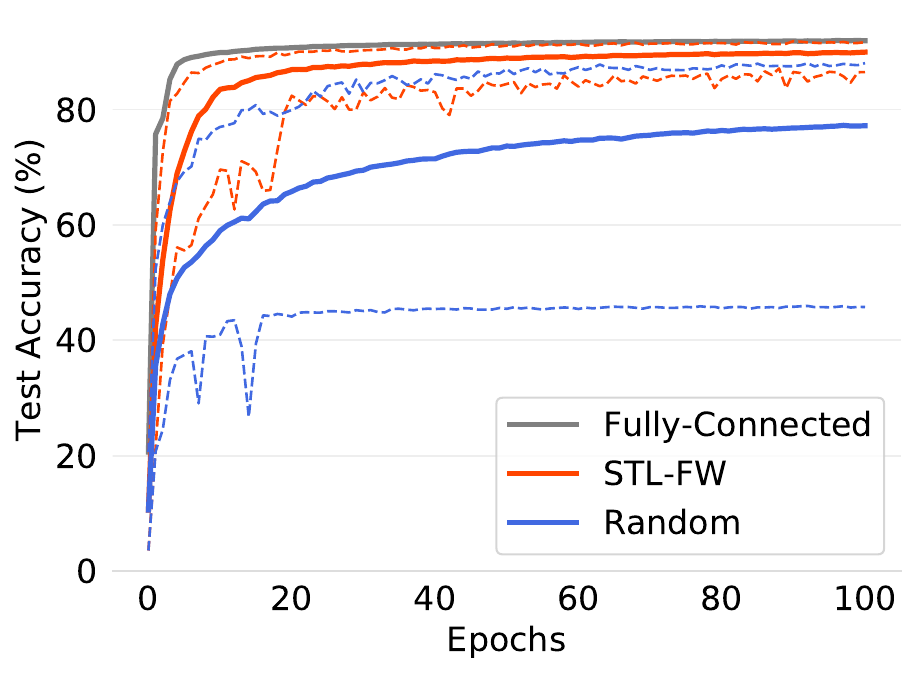}
        {\small MNIST $d_{\max}=2$}
    \end{minipage}\hfill
    \begin{minipage}[t]{\ratiorwexp\linewidth}
        \centering
        \includegraphics[height=\ratiob\linewidth]{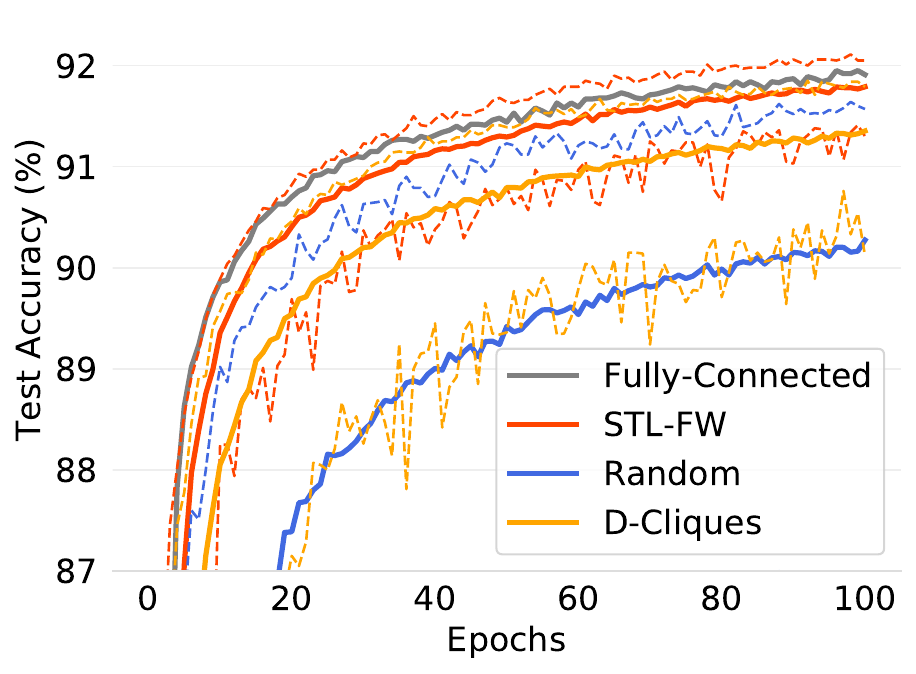}
        {\small MNIST $d_{\max}=5$ }
    \end{minipage}
    \begin{minipage}[t]{\ratiorwexp\linewidth}
        \centering
        \includegraphics[height=\ratiob\linewidth]{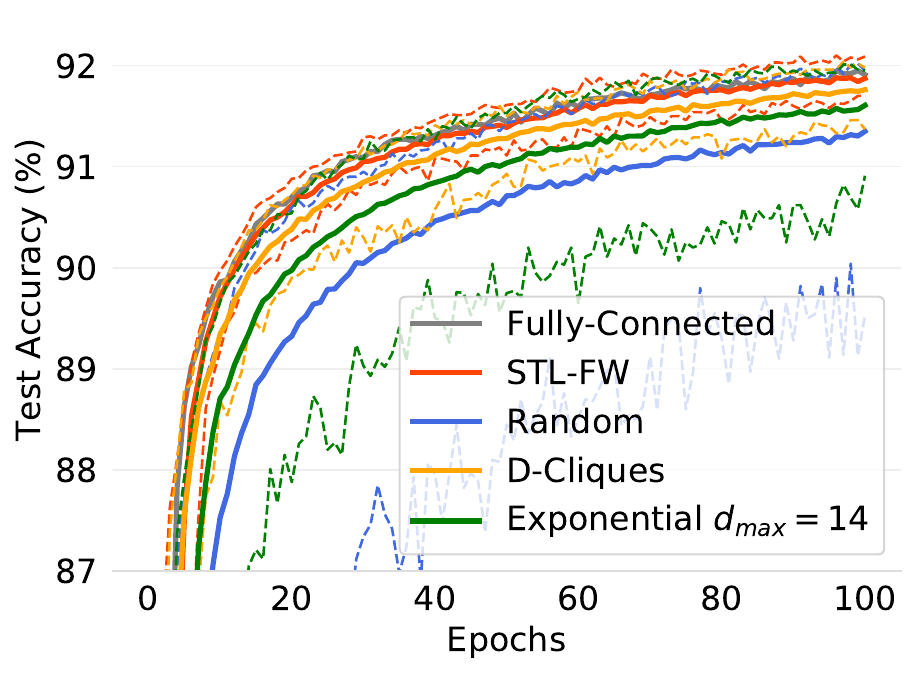}
        {\small MNIST $d_{\max}=10$ }
    \end{minipage}
        \begin{minipage}[t]{\ratiorwexp\linewidth}
        \centering
        \includegraphics[height=\ratiob\linewidth]{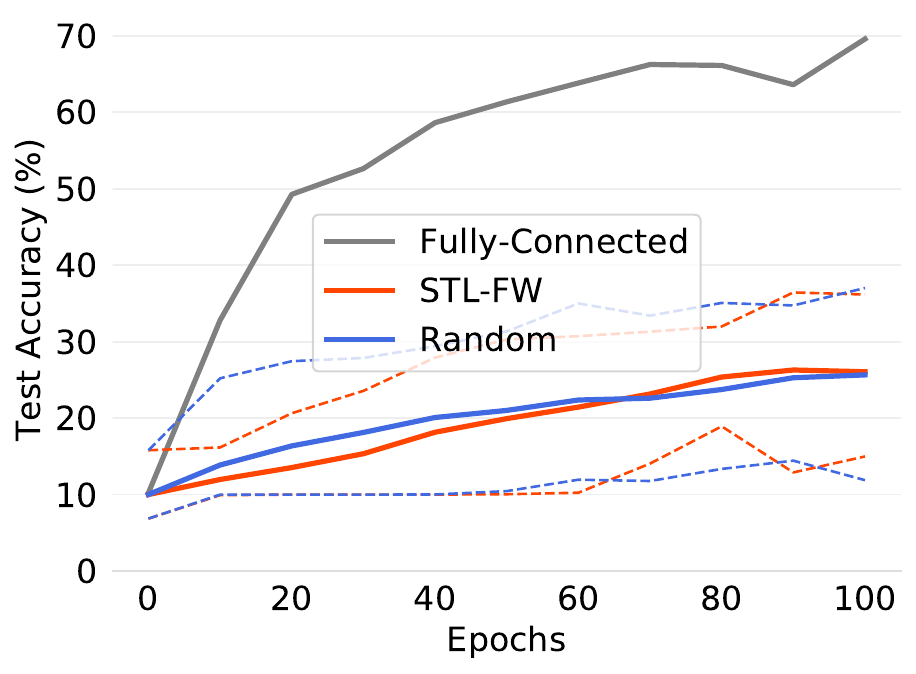}
        {\small CIFAR10 $d_{\max}=2$}
    \end{minipage}\hfill
    \begin{minipage}[t]{\ratiorwexp\linewidth}
        \centering
         \includegraphics[height=\ratiob\linewidth]{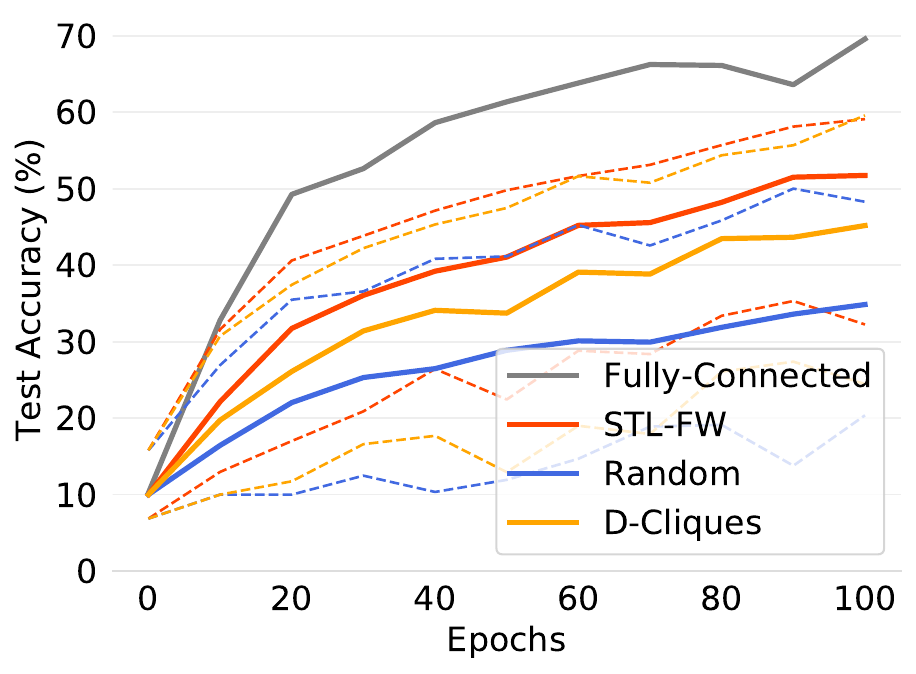}
        {\small CIFAR10 $d_{\max}=5$ }
    \end{minipage}
    \begin{minipage}[t]{\ratiorwexp\linewidth}
        \centering
        \includegraphics[height=\ratiob\linewidth]{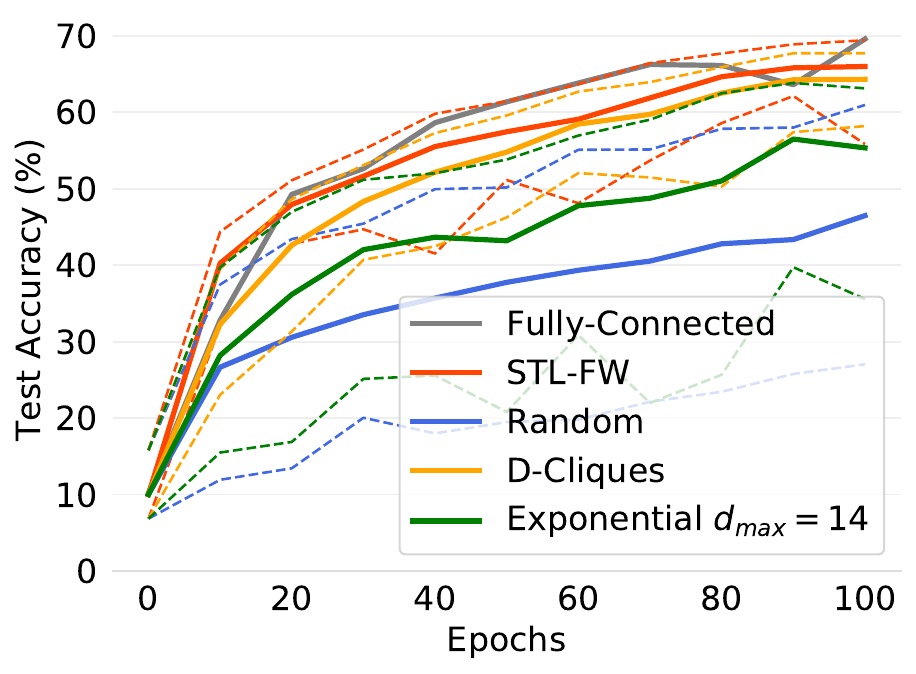}
        {\small CIFAR10 $d_{\max}=10$ }
    \end{minipage}
    \caption{Convergence of D-SGD with STL-FW (our approach) and alternative
    topologies on real datasets under different communication
 budgets. The fully connected graph induces intractable
 communication costs but gives a performance
 upper bound, while the exponential graph is shown for $d_{\max} = 10$ but
 exceeds this budget.}
 \label{fig:real_experiments}
\end{figure*}

\looseness=-1 \textbf{Statistical setup.} We generalize the mean estimation
objective of
Example
\ref{ex:ring} with $K=10$ clusters and $n=100$ nodes, with exactly $10$ nodes associated to each cluster. 
Each cluster is associated with a specific Gaussian
distribution,
which corresponds to a class in the statistical framework described in
Section~\ref{sec:framework_label_skew}.
The variance
of the $K$ distributions is the same ($
\tilde{\sigma}^2 = 1$) but their means are
evenly spread
over $[-m,m]$. Thus, $m\geq 0$ controls the heterogeneity of the problem (the
bigger $m$, the more heterogeneous the setup).
We can analytically compute all numerical constants
introduced throughout the paper.
Unless otherwise noted, $\lambda$ is set to $\sigma^2/KB$
 where $\sigma^2=4\tilde{\sigma}^2$ and
 $B=4m^2$.

\textbf{Competitor.} For a fixed budget $d_{\text{max}}$, we compare the topology learned
by STL-FW to a random $d_{\text{max}}$-regular graph with uniform weights $
\frac{1}{d_{\text{max}}+1}$. This graph is independent of the data but
 has good mixing parameter $p$ (every node will have exactly $b$
 neighbors, with uniform weights). 
 We use a fixed step-size for D-SGD, which is tuned separately for each 
 topology in the interval $[0.001,1]$.

\textbf{Results.} We first study the behavior of our topology learning
algorithm. As seen in Figure~\ref{fig:experiments_synth}(a), the objective
function $g(W^{(l)})$ decreases quickly in the first iterations with a clear
elbow at $l=9$ iterations. This is because we have $K=10$
``classes'', hence $9$ neighbors are sufficient to compensate for
label skew. We also see that decreasing $g$ successfully decreases
the two key quantities that affect the convergence of D-SGD and are upper
bounded by $g$:
the bias term in Equation \eqref{eq:H_bias_variance} (which does not depend on $\theta$ in this setup and can therefore be computed exactly) 
%
and the mixing parameter $1-p$ (which continues to decrease beyond $l=9$). \looseness =-1

Figure \ref{fig:experiments_synth}(b, c) shows that the topology learned by
STL-FW indeed translates into faster convergence for D-SGD than with the
 random (but well-connected) topology in data heterogeneous settings. This is
 especially striking when looking at best and worst-case errors across
 nodes (dashed lines). For a low budget ($d_{\text{max}}=3$), D-SGD with our
 topology
 remains slightly impacted by heterogeneity. But remarkably, for $d_{
 \text{max}}=9$, our topology makes D-SGD completely insensitive to
 increasing data heterogeneity. This
 observation is consistent with the elbow observed at $l=d_{\text{max}}=9$ in Figure~\ref{fig:experiments_synth}(a).
 In Appendix~\ref{sec:statistics_topologies}, we provide basic statistics on
 the topologies obtained for the two budgets $d_{\text{max}}=3$ and $d_{
 \text{max}}=9$. We can see in particular that the topologies obtained
 with
 STL-FW are $d_{\text{max}}$-regular (like the random graph) but have much
 lower bias (i.e., the distribution of labels in the neighborhood of nodes is
 closer to the global distribution). As expected, this bias is equal to $0$
 for $d_{\text{max}}=9$.

\subsection{Experiments on Real Datasets}
\label{sec:real-datasets}


\textbf{Setup.} We follow the experimental
setup in \cite{bellet2021d} and consider two classification tasks: a linear
model
 on MNIST \citep{deng2012mnist} and a Group Normalized LeNet~\citep{quagmire}
 on CIFAR10 \citep{krizhevsky2009learning}. In both cases, we partition the dataset on 100 nodes using the scheme introduced in~\cite{mcmahan2017communication}, i.e. 
on average, nodes will have examples of two classes, but may have only 1 and
 up to 4. We re-use the hyper-parameters from \cite{bellet2021d}: a learning
 rate of 0.1 and batch size of 128 for MNIST, and a learning rate of 0.002 and
 batch size of 20 for CIFAR10. Using D-SGD, we compare the topology learned
 with our approach STL-FW to other fixed
 topologies: (1) a \textit{fully-connected graph} ($d_{\max} = 99$), which
 exhibits the fastest convergence speed but is impractical, (2) a 
 \textit{random graph} with the
 same communication budget as STL-FW, (3)
 \textit{D-Cliques}~\citep{bellet2021d}, also with
 the same budget, and (4) a
 deterministic \textit{exponential graph} promoted in recent
 work~\citep{ying2021exponential} ($d_{\max} = 14$). 
 Note that all competing topologies are data-independent, except D-Cliques.
 To have a fair comparison, we use standard D-SGD without algorithmic
 modifications like ``clique-averaging'' introduced by 
\cite{bellet2021d}.
  For all experiments with STL-FW, we use $\lambda=0.1$ since, remarkably, its
 value does not significantly change the results (see Figure~
\ref{fig:lambda-effect} in Appendix~\ref{sec:additional-experiments}). 

\textbf{Results.} Figure~\ref{fig:real_experiments} shows our results for
varying communication budget $d_{max}$: small ($2$), medium ($5$) and large 
($10$). On MNIST, STL-FW makes convergence faster than all competitors and
quickly matches the speed of the fully-connected topology as the
budget $d_{\max}$ increases.
Remarkably, STL-FW is already showing good performance at $d_{\max} = 2$,
which is a very small budget that the other topologies (except the random one)
cannot handle. As expected, the two data-dependent topologies (D-Cliques
and STL-FW) outperform the random topologies, including the exponential graph
which has better connectivity ($d_{\max}=14$) but does
not compensate for the heterogeneity. The fact that STL-FW improves over
D-Cliques can be explained by the fact that D-Cliques only compensate the
heterogeneity (the bias term in Equation \eqref{eq:H_bias_variance}) without consideration for the
overall connectivity (the variance term in Equation 
\eqref{eq:H_bias_variance}). This is illustrated in the tables of Appendix \ref{sec:statistics_topologies}. 

On CIFAR10, we see that $d_{\max} = 2$ is not sufficient to reach good
performance. This can be explained by the increased
complexity of the problem (non-convex objective with a deep model), requiring
larger communication budgets. This is in line with empirical results in
prior work \citep{kong2021consensus}.
However, with slightly larger budgets i.e. $d_{\max} = 5$ and $10$ ($d_{\max} = 3$ in Fig.~\ref{fig:dmax-3}, App.~\ref{sec:additional-experiments}), performance
improves and the results are consistent with those on MNIST: STL-FW
outperforms other sparse topologies and comes close to the performance of the
fully connected topology for $d_{\max} = 10$.
Overall, STL-FW provides better convergence speed than all
tractable alternatives, with the additional ability to operate in low
communication regimes (unlike D-Cliques and the exponential graph).

\section{Conclusion}
\label{sec:conclu}

This paper addressed two important open problems in decentralized
 learning. First, thanks to our new notion of neighborhood
 heterogeneity, we characterized the joint effect of the
 topology and the data heterogeneity in the convergence rate of D-SGD. Our
 results show that, if chosen appropriately, the topology can compensate for
 the heterogeneity and speed up convergence. %
Second, we tackled the problem of learning a good topology under data
heterogeneity. To the best of our knowledge, our work is the first to provide
a principled and data-dependent approach, with explicit
control on the trade-off between the communication costs and the convergence
speed of D-SGD.
We believe that our work paves the way for the design of other
data-dependent topology learning techniques. One may for instance investigate
different types of heterogeneity (beyond label skew), different
knowledge assumptions (e.g., not knowing the proportions), and dynamic
learning of the topology. We can also envision fully decentralized
and privacy-preserving versions.




\bibliographystyle{apalike} 
\bibliography{biblio}

\newpage
\onecolumn

\appendix


\addcontentsline{toc}{section}{Appendix}
\section*{Appendix}

\section{Details on Example~\ref{ex:ring}}
\label{app:examples}

In this section, we provide more details on Example~\ref{ex:ring} (Section 
\ref{sec:our_quantity}) by giving the exact parametrization. Recall that
we want to find an example where Assumption \ref{ass:heterogeneityA} is not
verified while Assumption \ref{ass:new} is.

Let us consider $n$ nodes with $n$ an even number. For all $i=1,\ldots,n$,
assume
$Z_i\sim \mathcal{N}(m,\tilde{\sigma}^2)$ if $i$ is odd and $Z_i\sim 
\mathcal{N}(-m,\tilde{\sigma}^2)$ if $i$ is even. Assume further that $
\tilde{\sigma}^2<+\infty$ but $m>0$ can be asymptotically large. For all
$i=1,\ldots,n$ we fix $F_i(\theta,Z_i) = (\theta - Z_i)^2$, which
corresponds to a simple mean estimation objective.

Consider a fixed mixing matrix $W$ associated with a ring topology that
alternates between the two distributions. Specifically, for $i=2,\ldots, n-1$ and $j=1,\ldots,n$, we fix the
weights as follows:
\begin{equation*}
    W_{ij} = \left\{
        \begin{array}{ll}
          \frac{1}{2} \hspace{1em} \text{if}\hspace{0.5em}  j=i~,\\
          \frac{1}{4} \hspace{1em} \text{if} \hspace{0.5em} j=i+1  \hspace{0.5em} \text{or} \hspace{0.5em} j=i-1~,\\
          0 \hspace{1em} \text{otherwise}~.
        \end{array}
      \right.
\end{equation*}
Moreover, we fix $W_{11}=W_{nn}=\frac{1}{2}$ and $W_{1n}=W_{n1}=\frac{1}{4}$.

With such parametrization we have $\nabla F_i(\theta,Z_i) = 2(\theta-Z_i)$ and
therefore $\nabla f_i(\theta) = 2(\theta-m)$ if $i$ is odd and $\nabla f_i
(\theta) = 2(\theta+m)$ if $i$ is even. Moreover, the gradient of the global objective is $\nabla f(\theta)=\frac{1}{n}\sum_i\nabla f_i(\theta) = 2\theta$ and the neighborhood averaging $\sum_jW_{ij}\nabla f_j(\theta) = 2\theta$ for all $i$.

We first verify that Assumptions \ref{ass:variance} is satisfied:

\begin{equation*}
    \EE \left[\left(\nabla F_i(\theta,Z_i)- \nabla f_i(\theta)\right)^2\right] = \EE \left[4\left(Z_i - \EE Z_i\right)^2\right] = 4\tilde{\sigma}^2 < \infty~.
\end{equation*}

Let us now find a bound $\bar{\tau}^2$ on the neighborhood heterogeneity.
Using a bias-variance decomposition, we have:
\begin{align*}
    H(\theta) & = \frac{1}{n}\sum_{i=1}^n\EE\Bigg(\sum_{j=1}^nW_{ij}\nabla F_j(\theta) - \frac{1}{n}\sum_{j=1}^n\nabla F_j(\theta)\Bigg)^2 \\
    & =   \frac{1}{n}\sum_{i=1}^n\Big(\sum_{j=1}^nW_{ij}\nabla f_j(\theta) - \frac{1}{n}\sum_{j=1}^n\nabla f_j(\theta)\Big)^2 +  \frac{1}{n}\sum_{i=1}^n\EE\Big(\sum_{j=1}^n(W_{ij}-\frac{1}{n})(\nabla f_j(\theta) - \nabla F_j(\theta))\Big)^2 \\
    & =   \frac{1}{n}\sum_{i=1}^n\left(2\theta - 2\theta\right)^2 +  \frac{1}{n}\sum_{i=1}^n\sum_{j=1}^n(W_{ij}-\frac{1}{n})^2\EE(\nabla f_j(\theta) - \nabla F_j(\theta))^2 \\
    & = 0 + 4\tilde{\sigma}^2 \frac{1}{n}\sum_{i=1}^n\sum_{j=1}^n(W_{ij}-\frac{1}{n})^2 \leq 4\tilde{\sigma}^2~.
\end{align*}

The third equality was obtained thanks to the fact that $\EE[\nabla f_j
(\theta) - \nabla F_j(\theta)]=0$. This result shows that Assumption 
\ref{ass:new} is verified with $\bar{\tau}^2 = 4\tilde{\sigma}^2<\infty$.

On the contrary, since $m$ can be arbitrary large, Assumption 
\ref{ass:heterogeneityA} is not verified. Indeed:

\begin{align*}
    \frac{1}{n}\sum_{i=1}^n\left(\nabla f_i(\theta) - \frac{1}{n}\sum_{j=1}^n\nabla f_j(\theta)\right)^2 & =  \frac{1}{n}\sum_{i=1}^n (2m)^2 \\
    & = \frac{4m^2}{n} \triangleq \bar{\zeta}^2 \underset{m\rightarrow \infty}{\longrightarrow} + \infty~.
\end{align*}

\textbf{Remark.} 
At first sight, one may wonder why the local variance term $
\tilde{\sigma}^2$ appears in $\bar{\tau}^2$ but not in $\bar{\zeta}^2$. This is
because we chose to define neighborhood heterogeneity in expectation with
respect to the pointwise loss functions $F_1,\dots,F_n$, resulting in a
bias-variance decomposition (see Eq.~\ref{eq:H_bias_variance}) which is the
relevant quantity to optimize when learning the topology in
Section~\ref{sec:graph_learn}. In contrast, following the convention used
in previous work, local heterogeneity is defined
with respect to the local objectives $f_1,\dots,f_n$ and thus only measures a
bias term, while the variance term is accounted separately by
Assumption~\ref{ass:variance}. Since the variance terms are the same in
both settings, the difference is in how the bias term is measured (at the node level or at the neighborhood level):  in the
example above, it is equal to $\frac{4m^2}{n}$ for local heterogeneity while
it is equal to $0$ for neighborhood heterogeneity (see the above calculation
of $H(\theta)$).

\section{Proof of Theorem \ref{thme:main}}

\label{app:proof_thme}

\subsection{Notations and Overview}

We start by re-writing the updates of D-SGD (Algorithm~\ref{alg:d-sgd}) in
matrix form.

Let $\Theta^{(t)}\triangleq \left(\theta_1^{(t)}, \ldots, \theta_n^{
(t)}\right) \in \RR^{d\times n}$ be the matrix that contains the parameter
vectors of all nodes at time $t$. Denote by $\nabla F(\Theta^{(t)},Z^{(t)})
\triangleq \left(\nabla F_1(\theta^{(t)}_1,Z^{(t)}_1), \ldots, \nabla F_n
(\theta^{(t)}_n,Z^{(t)}_n)\right)\in \RR^{d\times n}$ the matrix containing
all stochastic gradients at time $t$. The D-SGD update at time $t$ can
then be written as:

$$\Theta^{(t+1)} = \left(\Theta^{(t)}-\eta_t \nabla F(\Theta^{(t)},Z^{(t)})\right)W^{(t)\transpose}~.$$

In the following, we denote $\overline{\Theta}^{(t)} \triangleq \left(\bar{\theta}^{(t)}, \ldots, \bar{\theta}^{(t)}\right)  = \Theta^{(t)}\cdot\frac{1}{n}\1\1^\transpose$.

The proof follows the classical steps found in the literature (see
e.g. \cite{koloskova20,neglia2020decentralized}). The main difference resides
in how the consensus term $\|\Theta^{(t)}-\overline{\Theta}^{
(t)}\|_F^2$ is controlled across iterations (Lemma 
\ref{lemma:consensus_control}). The proof is organized as follows.

\textbf{Convex case.}
\begin{enumerate}
\item Lemma \ref{lemma:descent} provides a descent recursion that
allows to control the decreasing of the term  $\left\|\bar{\theta}^{(t)} - \theta^{\star}\right\|^2$. The proof closely follows the one of \cite{koloskova20,neglia2020decentralized}.

\item In Lemma \ref{lemma:consensus_control}, the consensus term
$\|\Theta^{(t)}-\overline{\Theta}^{(t)}\|_F^2$, which appears in the result
of Lemma \ref{lemma:descent}, is upper-bounded. The resulting upper-bound
exhibits
our new quantity $\bar{\tau}^2$ (an upper bound on neighborhood
heterogeneity).
\item Corollary \ref{cor:cons_rec} uses the previous lemma to bound $\frac{1}
{T+1}\sum_{t=0}^T\|\Theta^{(t)}-\overline{\Theta}^{(t)}\|_F^2$.

\item Lemma \ref{lemma:rate_T} provides an upper-bound on the error
term with the following form:
\begin{equation*}
    \frac{1}{T+1}\sum_{t=0}^T\EE (f(\bar{\theta}^{(t)}) - f^\star) \leq 2\left(\frac{br_0}{T+1}\right)^{\frac{1}{2}} + 2e^{\frac{1}{3}}\left(\frac{r_0}{T+1}\right)^{\frac{2}{3}}+\frac{dr_0}{T+1},
\end{equation*}
where $b=\frac{\bar{\sigma}^2}{n}$, $e=\frac{36L\bar{\tau}^2}{p^2}$, $d=\frac{8L}{p}$ and $r_0=\|\theta^{(0)}-\theta^\star\|_2^2$.

\item To get the final rate of Theorem \ref{thme:main}, it suffices to find
$T$ such that each term in the right-hand side of the previous equation in bounded by $\frac{\varepsilon}{3}$.

\begin{itemize}
    \item $2\left(\frac{br_0}{T+1}\right)^{\frac{1}{2}} \leq \frac{\varepsilon}{3}  \Longleftrightarrow \frac{36br_0}{\varepsilon^2} \leq T+1 \Longleftrightarrow \frac{36\bar{\sigma}^2r_0}{n \varepsilon^2} \leq T+1$,
    \item  $2e^{\frac{1}{3}}\left(\frac{r_0}{T+1}\right)^{\frac{2}{3}} \leq \frac{\varepsilon}{3}  \Longleftrightarrow \frac{e^\frac{1}{2}6^\frac{3}{2}r_0}{\varepsilon^\frac{3}{2}} \leq T + 1  \Longleftrightarrow \frac{6^\frac{5}{2}\sqrt{L}\bar{\tau}r_0}{p \varepsilon^\frac{3}{2}} \leq T + 1 $,
    \item $\frac{dr_0}{T+1} \leq \frac{\varepsilon}{3}  \Longleftrightarrow \frac{3dr_0}{\varepsilon} \leq T+1 \Longleftrightarrow \frac{24Lr_0}{p\varepsilon} \leq T+1$.
\end{itemize}

In particular, it suffices to take 

\begin{equation*}
    T \geq \frac{36\bar{\sigma}^2r_0}{n \varepsilon^2} + \frac{89\sqrt{L}\bar{\tau}r_0}{p \varepsilon^\frac{3}{2}} + \frac{24Lr_0}{p\varepsilon} = \calO\left(\frac{\bar{\sigma}^2}{n \varepsilon^2} + \frac{\sqrt{L}\bar{\tau}}{p \varepsilon^\frac{3}{2}} + \frac{L}{p\varepsilon}\right)r_0~,
\end{equation*}

in order to have all three terms bounded by $\frac{\varepsilon}{3}$, and obtain the final result.
\end{enumerate}

\textbf{Non-convex case.} The proof is similar to the convex one: it only
differs in the descent lemmas that are used.

\begin{enumerate}
  \item Lemma \ref{lemma:descentNC} provides the descent lemma for the non-convex scenario.
  \item The consensus term is bounded using the same results as in the convex case, i.e., with Lemma \ref{lemma:consensus_control} and Corollary \ref{cor:cons_rec}.
  \item  Lemma \ref{lemma:rate_TNC} provides an upper-bound on the error
  term with the following form:
  \begin{equation*}
    \frac{1}{T+1}\sum_{t=0}^T\EE \Big\| \nabla f(\bar{\theta}^{(t)})  \Big\|_2^2 \leq 2\left(\frac{4bf_0}{T+1}\right)^{\frac{1}{2}} + 2e^{\frac{1}{3}}\left(\frac{4f_0}{T+1}\right)^{\frac{2}{3}}+\frac{4df_0}{T+1},
\end{equation*}
where $b=\frac{2L\bar{\sigma}^2}{n}$, $e=\frac{96L^2\bar{\tau}^2}{p^2}$, $d=\frac{8L}{p}$ and $f_0=f(\theta^{(0)}) - f^\star$.

  \item We bound in each term of the previous equation by $\frac{\varepsilon}{3}$ and get the sufficient condition:
  
  \begin{equation*}
    T \geq \frac{288L\bar{\sigma}^2f_0}{n \varepsilon^2} + \frac{576L\bar{\tau}f_0}{p \varepsilon^\frac{3}{2}} + \frac{96Lf_0}{p\varepsilon} = \calO\left(\frac{L\bar{\sigma}^2}{n \varepsilon^2} + \frac{L\bar{\tau}}{p \varepsilon^\frac{3}{2}} + \frac{L}{p\varepsilon}\right)f_0~.
\end{equation*}

\end{enumerate}

\subsection{Preliminaries and Useful Results}

\begin{property}[Averaging preservation]
\label{proper:mix_average} Let $W\in \RR^{n \times n}$ be a mixing matrix and $\Theta$ be any matrix in $\RR^{d\times n}$. Then, $W$ preserves averaging:
    \begin{equation}
        \label{eq:preserve_ave}
        (\Theta W)\frac{\1\1^\transpose}{n} = \Theta\frac{\1\1^\transpose}{n} = \overline{\Theta}
    \end{equation}
\end{property}

\begin{property}[Implications of $L$-smoothness and convexity]
~
\begin{itemize}
    \item If we assume convexity, we have for all
    $i\in \llbracket 1,\ldots, n\rrbracket$:
\begin{equation}
    \label{eq:grad_conv}
    \langle   \nabla f_i(\tilde{\theta})   ,  \tilde{\theta}   -   \theta   \rangle   \geq f_i( \tilde{\theta})   -   f_i(\theta).
\end{equation}
    \item Under Assumption \ref{ass:smoothness} ($L$-smoothness), it holds for
    all $i\in \llbracket 1,\ldots, n\rrbracket$:
\begin{equation}
    \label{eq:Lip_F_i}
    F_i(\theta,Z) \leq F_i(\tilde{\theta},Z) + \langle   \nabla F_i(\tilde{\theta},Z)   ,  \theta  - \tilde{\theta} \rangle   +   \frac{L}{2} \|\theta  - \tilde{\theta}\|_2^2, \quad \quad \quad \forall \theta, \tilde{\theta} \in \RR^d, Z\in\theta_i.
\end{equation}
Taking the expectation of the previous equation, we also
have:
\begin{equation}
    \label{eq:Lip_f_i}
    f_i(\theta) \leq f_i(\tilde{\theta}) + \langle   \nabla F(\tilde{\theta})   ,  \theta  - \tilde{\theta} \rangle   +   \frac{L}{2} \|\theta  - \tilde{\theta}\|_2^2, \quad \quad \quad \quad \forall \theta, \tilde{\theta} \in \RR^d.
\end{equation}
\item If we further assume that the $F_i$'s are convex, Assumption 
\ref{ass:smoothness} also implies $\forall \theta, \tilde{\theta} \in \RR^d$,$ Z\in\theta_i$:
\begin{align}
    \|\nabla f_i(\theta) - \nabla f_i(\tilde{\theta})\|_2 & \leq L \|\theta  - \tilde{\theta}\|_2, \label{eq:f_i_smooth}\\ 
    \|\nabla f_i(\theta) - \nabla f_i(\tilde{\theta})\|^2_2 & \leq 2L \left(f_i(\theta) - f_i(\tilde{\theta}) - \big\langle   \nabla f_i(
    \tilde{\theta})   ,  \theta  - \tilde{\theta} \big\rangle\right), \label{eq:f_i_smooth2} \\
    \|\nabla F_i(\theta,Z) - \nabla F_i(\tilde{\theta},Z)\|^2_2 & \leq 2L
    \left(F_i(\theta,Z) - F_i(\tilde{\theta},Z) - \big\langle   \nabla F_i(
    \tilde{\theta},Z)   ,  \theta  - \tilde{\theta} \big\rangle\right). 
    \label{eq:F_i_smooth2}
\end{align}
\end{itemize}

\end{property}

These results can be found in many convex optimization books and papers, e.g. in \cite{bubeck2014convex}.

\begin{property}[Norm inequalities]
~
    \begin{itemize}
        \item For a set of vectors $\{a_i\}_{i=1}^n$ such that $a_i\in\RR^{d}$,

        \begin{equation}
            \label{eq:norm_sum}
            \Big\|\sum_{i=1}^n a_i\Big\|_2^2 \leq n\sum_{i=1}^n
            \left\|a_i\right\|_2^2.
        \end{equation}
        \item For two vectors $a,b\in\RR^d$,
        \begin{equation}
            \label{eq:norm_alpha}
            \|a+b\|_2^2 \leq (1+\alpha)\|a\|_2^2 + (1+\alpha^{-1})\|b\|_2^2, \quad \quad \quad \forall \alpha>0.
        \end{equation}
        \item For two vectors $a,b\in\RR^d$,
        \begin{equation}
            \label{eq:inner_alpha}
            2\langle a, b \rangle \leq \alpha\|a\|_2^2 + \alpha^{-1}\|b\|_2^2, \quad \quad \quad \forall \alpha>0.
        \end{equation}
    \end{itemize}
\end{property}

\subsection{Needed Lemmas}


In the following we denote by $\mathcal{F}_t = \sigma (Z^{(k)} | k \leq t)$ the natural filtration with respect to $Z^{(t)} = (Z_1^{(t)},\ldots,Z_n^{(t)})$. Remark that $\forall i = 1,\ldots, n$ the iterates $\theta_i^{(t+1)}$ and $\bar{\theta}^{(t+1)}$ are in particular $\calF_{t}$-measurable.


\begin{lemma}[Descent Lemma - Convex case]
\label{lemma:descent} Consider the setting of
Theorem \ref{thme:main} and let $\eta_t\leq\frac{1}{4L}$, then we almost surely have:
\begin{equation}
    \EE_{Z^{(t)} | \calF_{t-1}} \left\|\bar{\theta}^{(t+1)} - \theta^{\star}\right\|^2 \leq \left\|\bar{\theta}^{(t)} - \theta^{\star}\right\|^2 + \frac{\eta_t^2\bar{\sigma}^2}{n} - \eta_t\left(f(\bar{\theta}^{(t)}) - f^\star\right) + \frac{3L}{2n}\eta_t \left\|\Theta^{(t)}-\overline{\Theta}^{(t)}\right\|_F^2,
\end{equation}
where $\EE_{Z^{(t)} | \calF_{t-1}}$ stands for the conditional expectation $\EE_{Z^{(t)}}[ \cdot | \calF_{t-1}]$.
\end{lemma}

\begin{proof}
The proof closely follows the one in \cite{koloskova20}. Using the recursion of D-SGD and since all mixing matrices are doubly stochastic and preserve the average (Proposition \ref{proper:mix_average}) we have:
\begin{align*}
    \|\bar{\theta}^{(t+1)}  &- \theta^{\star}\|^2  = \left\|\bar{\theta}^{(t)} - \frac{\eta_t}{n}\sum_{i=1}^n\nabla F_i(\theta_i^{(t)},Z_i^{(t)}) - \theta^{\star}\right\|^2 \\
    & = \left\|\bar{\theta}^{(t)} - \theta^{\star} -  \frac{\eta_t}{n}\sum_{i=1}^n\nabla f_i(\theta_i^{(t)}) +  \frac{\eta_t}{n}\sum_{i=1}^n\nabla f_i(\theta_i^{(t)}) - \frac{\eta_t}{n}\sum_{i=1}^n\nabla F_i(\theta_i^{(t)},Z_i^{(t)}) \right\|^2 \\
    & = \left\|\bar{\theta}^{(t)} - \theta^{\star} -  \frac{\eta_t}{n}\sum_{i=1}^n\nabla f_i(\theta_i^{(t)})  \right\|^2 + \eta_t^2\left\| \frac{1}{n}\sum_{i=1}^n\nabla f_i(\theta_i^{(t)}) - \frac{1}{n}\sum_{i=1}^n\nabla F_i(\theta_i^{(t)},Z_i^{(t)}) \right\|^2 \\
    & \quad \quad \quad + 2\left\langle \bar{\theta}^{(t)} - \theta^{\star} - 
    \frac{\eta_t}{n}\sum_{i=1}^n\nabla f_i(\theta_i^{(t)}) , \frac{\eta_t}{n}\sum_{i=1}^n\nabla f_i(\theta_i^{(t)}) - \frac{\eta_t}{n}\sum_{i=1}^n\nabla F_i(\theta_i^{(t)},Z_i^{(t)}) \right\rangle.
\end{align*}
Passing to the conditional expectation, the last term (the inner product) is
equal to $0$. This comes from the fact that $\EE_{Z_i^{(t)} | \calF_{t-1}} [\nabla F_i(\theta_i^{(t)},Z_i^{(t)})]= \nabla f_i(\theta_i^{(t)})$.
We therefore need to bound the first two terms in the conditional expectation.

The second one can easily be bounded using Assumption~\ref{ass:variance}:
\begin{align*}
    \eta_t^2\EE_{Z^{(t)} | \calF_{t-1}} \Big\| \frac{1}{n}\sum_{i=1}^n\nabla f_i(\theta_i^{(t)})  - \frac{1}{n}\sum_{i=1}^n & \nabla F_i(\theta_i^{(t)},Z_i^{(t)}) \Big\|^2 \\
    & = \frac{\eta_t^2}{n^2} \EE_{Z^{(t)} | \calF_{t-1}} \Big\| \sum_{i=1}^n(\nabla f_i(\theta_i^{(t)}) -\nabla F_i(\theta_i^{(t)},Z_i^{(t)}) )\Big\|^2 \\
    & =  \frac{\eta_t^2}{n^2}\sum_{i=1}^n \EE_{Z_i^{(t)} | \calF_{t-1}} \Big\| \nabla f_i(\theta_i^{(t)}) -\nabla F_i(\theta_i^{(t)},Z_i^{(t)}) \Big\|^2 \\
    & \overset{(A.\ref{ass:variance})}{\leq} \frac{\eta_t^2\bar{\sigma}^2}{n},
\end{align*}
where the second equality was obtained using the identity $\EE\left\|\sum_i Y_i\right\|_2^2 = \sum_i\EE \|Y_i\|_2^2$ when $Y_i$ are independent and $\EE Y_i = 0$.

Now that the second term is bounded, we can move to the first one. Because $\theta_i^{(t)}$ and $\bar{\theta}^{(t)}$ are $\calF_{t-1}$-measurable, we have 
\begin{align*}
    \EE_{Z^{(t)} | \calF_{t-1}}&\left\|\bar{\theta}^{(t)} - \theta^{\star} -  \frac{\eta_t}{n}\sum_{i=1}^n\nabla f_i(\theta_i^{(t)})  \right\|^2 = \left\|\bar{\theta}^{(t)} - \theta^{\star} -  \frac{\eta_t}{n}\sum_{i=1}^n\nabla f_i(\theta_i^{(t)})  \right\|^2 \\
    & = \left\|\bar{\theta}^{(t)} - \theta^{\star} \right\|^2 + \eta_t^2 \underbrace{\left\|\frac{1}{n}\sum_{i=1}^n\nabla f_i(\theta_i^{(t)})  \right\|^2}_{T_1} -\underbrace{2\eta_t\left\langle  \bar{\theta}^{(t)} - \theta^{\star} ,  \frac{1}{n}\sum_{i=1}^n\nabla f_i(\theta_i^{(t)}) \right\rangle}_{T_2}.
\end{align*}

In order to bound $T_1$, recall that by definition $\frac{1}{n}\sum_i\nabla f_i(\theta^\star) = 0$, therefore:

\begin{align*}
    T_1 & = \left\|\frac{1}{n}\sum_{i=1}^n(\nabla f_i(\theta_i^{(t)})   -  \nabla f_i(\bar{\theta}^{(t)})   +   \nabla f_i(\bar{\theta}^{(t)})   -   \nabla f_i(\theta^{\star}))\right\|^2 \\
    & \overset{(\ref{eq:norm_sum})}{\leq}  2\left\|\frac{1}{n}\sum_{i=1}^n(\nabla f_i(\theta_i^{(t)})   -  \nabla f_i(\bar{\theta}^{(t)})) \right\|^2 +   2\left\|\frac{1}{n}\sum_{i=1}^n(\nabla f_i(\bar{\theta}^{(t)})   -   \nabla f_i(\theta^{\star}))\right\|^2 \\
    & \overset{(\ref{eq:norm_sum})}{\leq} \frac{2}{n}\sum_{i=1}^n\left\|\nabla f_i(\theta_i^{(t)})   -  \nabla f_i(\bar{\theta}^{(t)}) \right\|^2 +   \frac{2}{n}\sum_{i=1}^n\left\|\nabla f_i(\bar{\theta}^{(t)})   -   \nabla f_i(\theta^{\star})\right\|^2 \\
    & \overset{(\ref{eq:f_i_smooth})(\ref{eq:f_i_smooth2})}{\leq} \frac{2L^2}{n}\sum_{i=1}^n\left\|\theta_i^{(t)}  -  \bar{\theta}^{(t)}\right\|^2      +    \frac{4L}{n}\sum_{i=1}^n    \left(f_i(\bar{\theta}^{(t)}) - f_i(\theta^{\star}) - \left\langle   \nabla f_i(\theta^{\star})   ,  \bar{\theta}^{(t)}  - \theta^{\star} \right\rangle\right) \\
    & = \frac{2L^2}{n}\sum_{i=1}^n\left\|\theta_i^{(t)}  -  \bar{\theta}^{(t)}\right\|^2    +    \frac{4L}{n}\sum_{i=1}^n    \left(f_i(\bar{\theta}^{(t)}) - f_i(\theta^{\star})\right)      -      4L\left\langle   \underbrace{\frac{1}{n}\sum_{i=1}^n \nabla f_i(\theta^{\star})}_{=0}   ,  \bar{\theta}^{(t)}  - \theta^{\star} \right\rangle \\
    & = \frac{2L^2}{n}\sum_{i=1}^n\left\|\theta_i^{(t)}  -  \bar{\theta}^{(t)}\right\|^2    +    4L\left(f(\bar{\theta}^{(t)}) - f^\star\right).
\end{align*}

Finally, we have to bound $T_2$:
\begin{align*}
    -T_2   & =    - \frac{2\eta_t}{n}\sum_{i=1}^n \left\langle  \bar{\theta}^{(t)} - \theta^{\star} , \nabla f_i(\theta_i^{(t)}) \right\rangle \\
    & = - \frac{2\eta_t}{n}\sum_{i=1}^n \left[\left\langle  \bar{\theta}^{(t)} - \theta_i^{(t)} , \nabla f_i(\theta_i^{(t)}) \right\rangle      +     \left\langle  \theta_i^{(t)} - \theta^{\star} , \nabla f_i(\theta_i^{(t)}) \right\rangle \right] \\
    & \overset{(\ref{eq:Lip_f_i})(\ref{eq:grad_conv})}{\leq} - \frac{2\eta_t}{n}\sum_{i=1}^n \left[ f_i(\bar{\theta}^{(t)})   -   f_i(\theta_i^{(t)})   -    \frac{L}{2}\|\bar{\theta}^{(t)}   -   \theta_i^{(t)}\|_2^2   +   f_i(\theta_i^{(t)})   -   f_i(\theta^{\star})\right] \\
    & = -2 \eta_t \left(f(\bar{\theta}^{(t)}) - f( \theta^\star)\right)   +   \frac{L\eta_t}{n} \sum_{i=1}^n \|\bar{\theta}^{(t)}   -   \theta_i^{(t)}\|_2^2 \\
    & =  -2 \eta_t \left(f(\bar{\theta}^{(t)}) - f^\star\right)   +   \frac{L\eta_t}{n} \|\overline{\Theta}^{(t)} - \Theta^{(t)}\|_F^2.
\end{align*}

Combining all previous results, we get:
\begin{align*}
    \EE_{Z^{(t)} | \calF_{t-1}} \left\|\bar{\theta}^{(t+1)} - \theta^{\star}\right\|^2 & \leq  \left\|\bar{\theta}^{(t)} - \theta^{\star} \right\|^2 + \frac{\eta_t^2\bar{\sigma}^2}{n} + \frac{L\eta_t}{n}\left(2L\eta_t+1\right)  \|\overline{\Theta}^{(t)} - \Theta^{(t)}\|_F^2 \\
    & \hspace{4cm} + 2\eta_t\left(2L\eta_t - 1\right) \left(f(\bar{\theta}^{(t)}) - f^\star\right).
\end{align*}

Since, by hypothesis, $\eta_t\leq\frac{1}{4L}$, we have $2L\eta_t+1\leq \frac{3}{2}$ and $2L\eta_t - 1 \leq -\frac{1}{2}$, which concludes the proof.
\end{proof}

\begin{lemma}[Descent Lemma - Non-convex case]
    \label{lemma:descentNC} Consider the setting of
    Theorem \ref{thme:main} and let $\eta_t\leq\frac{1}{4L}$, then we almost surely have:
    \begin{equation}
        \EE_{Z^{(t)} | \calF_{t-1}} f(\bar{\theta}^{(t+1)}) - f^\star \leq f(\bar{\theta}^{(t)}) - f^\star - \frac{\eta_t}{4}\|\nabla f(\bar{\theta}^{(t)})\|_2^2+ \frac{L^2}{n}\eta_t \left\|\Theta^{(t)}-\overline{\Theta}^{(t)}\right\|_F^2 + \frac{L\bar{\sigma}^2}{2n}\eta_t^2.
    \end{equation}

    \end{lemma}

    \begin{proof}
        The proof adapts the one of Lemma $10$ in \cite{koloskova20} to our setting.

\begin{align*}
    & \EE_{Z^{(t)} | \calF_{t-1}} f(\bar{\theta}^{(t+1)}) = \EE_{Z^{(t)} | \calF_{t-1}} f\Big( \bar{\theta}^{(t)} - \frac{\eta_t}{n}\sum_{i=1}^n\nabla F_i(\theta_i^{(t)},Z_i^{(t)})\Big) \\
    & \quad \overset{\eqref{eq:Lip_f_i}}{\leq} f(\bar{\theta}^{(t)}) - \EE_{Z^{(t)} | \calF_{t-1}} \Big\langle \nabla f(\bar{\theta}^{(t)})     ,     \frac{\eta_t}{n}\sum_{i=1}^n\nabla F_i(\theta_i^{(t)},Z_i^{(t)}) \Big\rangle \\
    & \pushright{+ \frac{L}{2}\EE_{Z^{(t)} | \calF_{t-1}}\Big\|\frac{\eta_t}{n}\sum_{i=1}^n\nabla F_i(\theta_i^{(t)},Z_i^{(t)})\Big\|_2^2 }\\
    & \quad = f(\bar{\theta}^{(t)}) \underbrace{ - \Big\langle \nabla f(\bar{\theta}^{(t)})     ,     \frac{\eta_t}{n}\sum_{i=1}^n\nabla f_i(\theta_i^{(t)}) \Big\rangle }_{\triangleq T_4} + \frac{L\eta_t^2}{2}   \underbrace{\EE_{Z^{(t)} | \calF_{t-1}}\Big\|\frac{1}{n}\sum_{i=1}^n\nabla F_i(\theta_i^{(t)},Z_i^{(t)})\Big\|_2^2}_{\triangleq T_5}
\end{align*}

Adding and subtracting $\eta_t \nabla f(\bar{\theta}^{(t)})$ in $T_4$, we have 

\begin{align*}
    T_4 & = - \eta_t \Big\|\nabla f(\bar{\theta}^{(t)})\Big\|_2^2 + \frac{\eta_t}{n}\sum_{i=1}^n     \Big\langle \nabla f(\bar{\theta}^{(t)})     ,     \nabla f_i(\bar{\theta}^{(t)}) - \nabla f_i(\theta_i^{(t)}) \Big\rangle \\
    & \overset{\eqref{eq:inner_alpha}, \alpha = 1}{\leq} - \eta_t \Big\|\nabla f(\bar{\theta}^{(t)})\Big\|_2^2 +  \frac{\eta_t}{2} \Big\|\nabla f(\bar{\theta}^{(t)})\Big\|_2^2 +  \frac{\eta_t}{2n} \sum_{i=1}^n \Big\|\nabla f_i(\bar{\theta}^{(t)})  - \nabla f_i(\theta_i^{(t)}) \Big\|_2^2 \\
    & \overset{\eqref{eq:f_i_smooth}}{\leq} -  \frac{\eta_t}{2} \Big\|\nabla f(\bar{\theta}^{(t)})\Big\|_2^2 +  \frac{L^2\eta_t}{2n}  \sum_{i=1}^n \Big\|\bar{\theta}^{(t)}  - \theta_i^{(t)} \Big\|_2^2~.
\end{align*}

Let us now bound the term $T_5$:

\begin{align*}
    T_5 &= \EE_{Z^{(t)} | \calF_{t-1}}\Big\|\frac{1}{n}\sum_{i=1}^n\nabla F_i(\theta_i^{(t)},Z_i^{(t)}) - \frac{1}{n}\sum_{i=1}^n\nabla f_i(\theta_i^{(t)}) + \frac{1}{n}\sum_{i=1}^n\nabla f_i(\theta_i^{(t)})\Big\|_2^2 \\
    &= \frac{1}{n^2}\EE_{Z^{(t)} | \calF_{t-1}}\Big\|\sum_{i=1}^n\nabla F_i(\theta_i^{(t)},Z_i^{(t)}) -\sum_{i=1}^n\nabla f_i(\theta_i^{(t)})\Big\|_2^2 + \Big\| \frac{1}{n}\sum_{i=1}^n\nabla f_i(\theta_i^{(t)})\Big\|_2^2 \\
    & \overset{(A.\ref{ass:variance})}{=} \frac{\bar{\sigma}^2}{n} + \Big\| \frac{1}{n}\sum_{i=1}^n\nabla f_i(\theta_i^{(t)}) - \nabla f(\bar{\theta}^{(t)}) + \nabla f(\bar{\theta}^{(t)})\Big\|_2^2 \\
    & \overset{\eqref{eq:norm_sum}}{\leq} \frac{\bar{\sigma}^2}{n} + 2\Big\| \frac{1}{n}\sum_{i=1}^n\nabla f_i(\theta_i^{(t)}) - \nabla f(\bar{\theta}^{(t)}) \Big\|_2^2 +  2\Big\|\nabla f(\bar{\theta}^{(t)})\Big\|_2^2 \\
    & \overset{\eqref{eq:norm_sum}}{\leq} \frac{\bar{\sigma}^2}{n} + \frac{2}{n} \sum_{i=1}^n \Big\| \nabla f_i(\theta_i^{(t)}) - \nabla f_i(\bar{\theta}^{(t)}) \Big\|_2^2 +  2\Big\|\nabla f(\bar{\theta}^{(t)})\Big\|_2^2 \\
    & \overset{\eqref{eq:f_i_smooth}}{\leq} \frac{\bar{\sigma}^2}{n} + \frac{2L^2}{n} \sum_{i=1}^n \Big\| \theta_i^{(t)} - \bar{\theta}^{(t)} \Big\|_2^2 +  2\Big\|\nabla f(\bar{\theta}^{(t)})\Big\|_2^2.
\end{align*}

Next, plugging $T_4$ and $T_5$ into the first inequality, we have:

\begin{align*}
    &\EE_{Z^{(t)} | \calF_{t-1}} f(\bar{\theta}^{(t+1)}) \\
    & \quad \leq f(\bar{\theta}^{(t)})   -  \eta_t\Big(\frac{1}{2} - L\eta_t\Big) \Big\|\nabla f(\bar{\theta}^{(t)})\Big\|_2^2 +  \Big(\frac{L^2\eta_t}{2n} + \frac{L^3\eta_t^2}{n}\Big) \left\|\Theta^{(t)}-\overline{\Theta}^{(t)}\right\|_F^2 + \frac{L\bar{\sigma}^2}{2n}\eta_t^2~.
\end{align*}
Since by hypothesis $\eta_t\leq \frac{1}{4L}$, we have $\frac{1}{2} - L\eta_t
\geq \frac{1}{4}$ and $\frac{L^2\eta_t}{2n} + \frac{L^3\eta_t^2}{n} \leq 
\frac{L^2\eta_t}{n}$, we therefore get

\begin{align*}
    \EE_{Z^{(t)} | \calF_{t-1}} f(\bar{\theta}^{(t+1)}) \leq f(\bar{\theta}^{(t)})   -  \frac{\eta_t}{4} \Big\|\nabla f(\bar{\theta}^{(t)})\Big\|_2^2 +  \frac{L^2\eta_t}{n}\left\|\Theta^{(t)}-\overline{\Theta}^{(t)}\right\|_F^2 + \frac{L\bar{\sigma}^2}{2n}\eta_t^2~.
\end{align*}
Subtracting each side of the equation by $f^\star$, we obtain the final result.
    \end{proof}

\begin{lemma}[Consensus Control]
\label{lemma:consensus_control}
Consider the setting of Theorem \ref{thme:main} and let $\eta_t\leq\frac{p}{8L}$, then:
    \begin{equation}
        \label{eq:consensus_control}
        \EE \left\|\Theta^{(t)}-\overline{\Theta}^{(t)}\right\|_F^2 \leq (1-\frac{p}{4})\EE \left\|\Theta^{(t-1)}-\overline{\Theta}^{(t-1)}\right\|_F^2 + \frac{6n\bar{\tau}^2}{p}\eta_{t-1}^2.
    \end{equation}
\end{lemma}

\begin{proof} For any $\alpha>0$, we have:
    \begin{align*}
        \EE  & \left\|\Theta^{(t)}-\overline{\Theta}^{(t)}\right\|_F^2  =  \EE \left\|\Theta^{(t)}\left(I - \frac{\1\1^\transpose}{n}\right)\right\|_F^2\\
        & = \EE \left\|   \left(\Theta^{(t-1)}-\eta_{t-1} \nabla F(\Theta^{(t-1)},Z^{(t-1)})\right)W^{(t-1)\transpose}   \left(I - \frac{\1\1^\transpose}{n}\right)  \right\|_F^2 \\
        & \overset{(\ref{eq:preserve_ave})}{=} \EE \left\|   \left(\Theta^{(t-1)}-\eta_{t-1} \nabla F(\Theta^{(t-1)},Z^{(t-1)})\right)   \left(W^{(t-1)\transpose} - \frac{\1\1^\transpose}{n}\right)  \right\|_F^2 \\
        & \overset{(\ref{eq:norm_alpha})}{\leq} (1+\alpha) \EE \left\|   \Theta^{(t-1)}\left(W^{(t-1)\transpose} - \frac{\1\1^\transpose}{n}\right)  \right\|_F^2   \\
        &  \pushright{+     (1+\alpha^{-1})\eta_{t-1}^2\underbrace{\EE \left\|   \nabla F(\Theta^{(t-1)},Z^{(t-1)})\left(W^{(t-1)\transpose} - \frac{\1\1^\transpose}{n}\right)  \right\|_F^2}_{T_3} }\\
        & \overset{(A.\ref{ass:rho})}{\leq} (1+\alpha) (1-p)\EE \left\|\Theta^{
        (t-1)}-\overline{\Theta}^{(t-1)}\right\|_F^2  + (1+\alpha^{-1})\eta_{t-1}^2T_3.
    \end{align*}

We now bound $T_3$ by relying on Assumption~\ref{ass:new}:
\begin{align*}
    T_3 & = \EE \Bigg\|   \left(\nabla F(\Theta^{(t-1)},Z^{(t-1)})    -    \nabla F(\overline{\Theta}^{(t-1)},Z^{(t-1)})    +      \nabla F(\overline{\Theta}^{(t-1)},Z^{(t-1)})  \right) \cdot \\
    & \hspace{10cm} \cdot \left(W^{(t-1)\transpose} - \frac{\1\1^\transpose}{n}\right)  \Bigg\|_F^2 \\
    & \overset{(\ref{eq:norm_sum})}{\leq} 2\EE \left\|   \left(\nabla F(\Theta^{(t-1)},Z^{(t-1)})    -    \nabla F(\overline{\Theta}^{(t-1)},Z^{(t-1)})     \right)\left(W^{(t-1)\transpose} - \frac{\1\1^\transpose}{n}\right)  \right\|_F^2     \\
    &  \pushright{ + 2\EE \left\|     \nabla F(\overline{\Theta}^{(t-1)},Z^{(t-1)})  \left(W^{(t-1)\transpose} - \frac{\1\1^\transpose}{n}\right)  \right\|_F^2 } \\
    & = 2\EE \left\|   \left(\nabla F(\Theta^{(t-1)},Z^{(t-1)})    -    \nabla F(\overline{\Theta}^{(t-1)},Z^{(t-1)})     \right)\left(W^{(t-1)\transpose} - \frac{\1\1^\transpose}{n}\right)  \right\|_F^2     \\
    &  \pushright{ +     2\sum_{i=1}^n\EE\left\|\sum_{j=1}^nW^{(t-1)}_{ij}\nabla F_j(\bar{\theta}^{(t-1)},Z_j^{(t-1)}) - \frac{1}{n}\sum_{j=1}^n\nabla F_j(\bar{\theta}^{(t-1)},Z_j^{(t-1)})\right\|_2^2 } \\
    & \overset{(\ref{eq:new})}{\leq} 2\EE \left\|   \left(\nabla F(\Theta^{(t-1)},Z^{(t-1)})    -    \nabla F(\overline{\Theta}^{(t-1)},Z^{(t-1)})     \right)\left(W^{(t-1)\transpose} - \frac{\1\1^\transpose}{n}\right)  \right\|_F^2 + 2n \bar{\tau}^2.
\end{align*}

For conciseness, we will denote $F_i(\theta_i^{(t-1)},Z_j^{(t-1)})$ by $F_i
(\theta_i^{(t-1)})$ and
$\nabla F
(\Theta,Z^{(t-1)})$ by $\nabla F(\Theta)$. Using Assumption
\ref{ass:rho}, we can bound the first term of the previous equation by:
\begin{align*}
    & \hspace{0.5cm} 2(1-p) \EE \left\|   \left(\nabla F(\Theta^{(t-1)})    -    \nabla F(\overline{\Theta}^{(t-1)})     \right) -  \left(\overline{\nabla F(\Theta^{(t-1)})    -    \nabla F(\overline{\Theta}^{(t-1)}) }    \right) \right\|_F^2 \\
    & \overset{(\ref{eq:norm_sum})}{\leq} 4(1-p)\left[\EE \left\|   \nabla F(\Theta^{(t-1)})    -    \nabla F(\overline{\Theta}^{(t-1)})    \right\|_F^2 + \EE \left\| \overline{\nabla F(\Theta^{(t-1)})    -    \nabla F(\overline{\Theta}^{(t-1)}) }    \right\|_F^2\right]\\
    & = 4(1-p)\times\\
    &\times \left[    \sum_{i=1}^n\Bigg(\EE \left\|    \nabla F_i(\theta_i^{(t-1)})  - \nabla F_i(\bar{\theta}^{(t-1)})     \right\|_2^2     +     \EE \Big\|    \frac{1}{n}\sum_{j=1}^n\left(\nabla F_j(\theta_j^{(t-1)})  - \nabla F_j(\bar{\theta}^{(t-1)})  \right)   \Big\|_2^2   \Bigg) \right] \\
    & \overset{(A.\ref{ass:smoothness})}{\leq} 4(1-p)\left[ L^2\sum_{i=1}^n         \EE \left\|    \theta_i^{(t-1)}  -  \bar{\theta}^{(t-1)}     \right\|_2^2             + \frac{n}{n^2} \EE \Big\|   \sum_{j=1}^n\left(\nabla F_j(\theta_j^{(t-1)})  - \nabla F_j(\bar{\theta}^{(t-1)})  \right)   \Big\|_2^2   \right] \\
    & \overset{(\ref{eq:norm_sum})}{\leq} 4(1-p)\left[ L^2\sum_{i=1}^n         \EE \left\|    \theta_i^{(t-1)}  -  \bar{\theta}^{(t-1)}     \right\|_2^2             +  \sum_{j=1}^n \EE \left\|  \nabla F_j(\theta_j^{(t-1)})  - \nabla F_j(\bar{\theta}^{(t-1)})    \right\|_2^2   \right]\\
    & \overset{(A.\ref{ass:smoothness})}{\leq} 4(1-p)\left[ L^2\sum_{i=1}^n         \EE \left\|    \theta_i^{(t-1)}  -  \bar{\theta}^{(t-1)}     \right\|_2^2             + L^2 \sum_{j=1}^n \EE \left\|    \theta_j^{(t-1)}  -  \bar{\theta}^{(t-1)}     \right\|_2^2   \right] \\
    & = 8(1-p)L^2\EE \left\|\Theta^{(t-1)}-\overline{\Theta}^{(t-1)}\right\|_F^2.
\end{align*}

Combining all previous results and setting $\alpha = \frac{p}{2}$ , we get:
\begin{align*}
    \EE  \left\|\Theta^{(t)}-\overline{\Theta}^{(t)}\right\|_F^2 & \leq (1+\alpha) (1-p)\EE \left\|\Theta^{(t-1)}-\overline{\Theta}^{(t-1)}\right\|_F^2  \\
    & \hspace{0.5cm} + 8(1+\alpha^{-1})(1-p)L^2\eta_{t-1}^2 \EE \left\|\Theta^{(t-1)}-\overline{\Theta}^{(t-1)}\right\|_F^2        +      2 (1+\alpha^{-1})\eta_{t-1}^2 n \bar{\tau}^2 \\
    & \leq \underbrace{(1+\frac{p}{2}) (1-p)}_{\leq 1 - \frac{p}{2}}\EE \left\|\Theta^{(t-1)}-\overline{\Theta}^{(t-1)}\right\|_F^2  \\
    & \hspace{1cm} + \underbrace{8(1+\frac{2}{p})(1-p)}_{\leq \frac{16}{p}}L^2\eta_{t-1}^2 \EE \left\|\Theta^{(t-1)}-\overline{\Theta}^{(t-1)}\right\|_F^2        +      \underbrace{2 (1+\frac{2}{p})}_{\leq \frac{6}{p}}\eta_{t-1}^2 n \bar{\tau}^2.
\end{align*}

Since by hypothesis we have $\eta_{t-1}\leq \frac{p}{8L}$, we can bound the second term and get:
\begin{align*}
    \EE  \left\|\Theta^{(t)}-\overline{\Theta}^{(t)}\right\|_F^2 & \leq \left(1-\frac{p}{2}+\frac{p}{4}\right)\EE \left\|\Theta^{(t-1)}-\overline{\Theta}^{(t-1)}\right\|_F^2 + \frac{6n\bar{\tau}^2}{p}\eta_{t-1}^2 \\
    & = \left(1-\frac{p}{4}\right)\EE \left\|\Theta^{(t-1)}-\overline{\Theta}^{(t-1)}\right\|_F^2 + \frac{6n\bar{\tau}^2}{p}\eta_{t-1}^2.
\end{align*}

\end{proof}

\begin{corollary}[Consensus recursion]
\label{cor:cons_rec} Consider the setting
of Theorem \ref{thme:main} and fix $\eta_t = \eta \leq\frac{p}{8L}$, we have:
    \begin{equation}
        \label{eq:cons_recursion}
        \frac{1}{T+1}\sum_{t=0}^T \EE \left\|\Theta^{(t)}-\overline{\Theta}^{(t)}\right\|_F^2 \leq \frac{24\eta^2n\bar{\tau}^2}{p^2}.
    \end{equation}
\end{corollary}

\begin{proof}
    Unrolling the expression (\ref{eq:consensus_control}) in Lemma 
    \ref{lemma:consensus_control} up to $t=0$, we have for all $t>0$: 
    \begin{align*}
        \EE \left\|\Theta^{(t)}-\overline{\Theta}^{(t)}\right\|_F^2 & \leq \left(1-\frac{p}{4}\right)^t    \underbrace{\left\|\Theta^{(0)}-\overline{\Theta}^{(0)}\right\|_F^2}_{ = 0}         +           \frac{6n\bar{\tau}^2}{p}\eta^2 \sum_{j=0}^{t-1} \left(1-\frac{p}{4}\right)^{j} \\
        &  =  \frac{6n\bar{\tau}^2}{p}\eta^2 \times \frac{1- \left(1-\frac{p}{4}\right)^{t}}{1- \left(1-\frac{p}{4}\right)} \\
        & \leq \frac{6\eta^2 n\bar{\tau}^2}{p} \times \frac{4}{p}\\
        & = \frac{24\eta^2n\bar{\tau}^2}{p^2}
    \end{align*}

    Summing and dividing by $T+1$, we get the final result.
\end{proof}

\begin{lemma}[Convergence rate with $T$ - Convex case]
\label{lemma:rate_T}  Consider the
setting of Theorem \ref{thme:main} in the convex case. There exists a constant stepsize $\eta\leq \eta_{\max}=\frac{p}{8L}$ such that
    \begin{equation}
        \frac{1}{T+1}\sum_{t=0}^T\EE (f(\bar{\theta}^{(t)}) - f^\star) \leq 2\left(\frac{br_0}{T+1}\right)^{\frac{1}{2}} + 2e^{\frac{1}{3}}\left(\frac{r_0}{T+1}\right)^{\frac{2}{3}}+\frac{dr_0}{T+1},
    \end{equation}
    where $b=\frac{\bar{\sigma}^2}{n}$, $e=\frac{36L\bar{\tau}^2}{p^2}$, $d=\frac{8L}{p}$ and $r_0=\|\theta^{(0)}-\theta^\star\|_2^2$.
\end{lemma}

\begin{proof}
    Thanks to the descent lemma  (Lemma~\ref{lemma:descent}), we almost surely have:

    \begin{align*}
        f(\bar{\theta}^{(t)}) - f^\star \leq \frac{1}{\eta}\Big(  \left\|\bar{\theta}^{(t)} - \theta^{\star}\right\|^2 - \EE_{Z^{(t)} | \calF_{t-1}} \left\|\bar{\theta}^{(t+1)} - \theta^{\star}\right\|^2 + \frac{\eta^2\bar{\sigma}^2}{n}  + \frac{3L}{2n}\eta \left\|\Theta^{(t)}-\overline{\Theta}^{(t)}\right\|_F^2   \Big),
    \end{align*}
    where all terms are $\calF_{t-1}$-measurable. Therefore, 

    \begin{align*}
        \EE (f(\bar{\theta}^{(t)}) - f^\star) \leq \frac{1}{\eta}\Big(  \EE \left\|\bar{\theta}^{(t)} - \theta^{\star}\right\|^2 - \EE \left\|\bar{\theta}^{(t+1)} - \theta^{\star}\right\|^2 + \frac{\eta^2\bar{\sigma}^2}{n}  + \frac{3L}{2n}\eta \EE \left\|\Theta^{(t)}-\overline{\Theta}^{(t)}\right\|_F^2   \Big),
    \end{align*}
    and summing up we get:

    \begin{align*}
        \frac{1}{T+1} & \sum_{t=0}^T\EE (f(\bar{\theta}^{(t)}) - f^\star) \\ 
        & \leq \frac{1}{\eta(T+1)}\sum_{t=0}^T\left( \EE \left\|\bar{\theta}^{(t)} - \theta^{\star}\right\|^2 - \EE \left\|\bar{\theta}^{(t+1)} - \theta^{\star}\right\|^2 + \frac{\eta^2\bar{\sigma}^2}{n}  + \frac{3L}{2n}\eta \EE \left\|\Theta^{(t)}-\overline{\Theta}^{(t)}\right\|_F^2  \right) \\
        & \leq   \frac{1}{\eta(T+1)}\left\|\theta^{(0)} - \theta^{\star}\right\|^2 + \frac{\eta\bar{\sigma}^2}{n} + \frac{3L}{2n}  \frac{1}{T+1}\sum_{t=0}^T \EE \left\|\Theta^{(t)}-\overline{\Theta}^{(t)}\right\|_F^2 \\
        & \overset{(\ref{eq:cons_recursion})}{\leq}    \frac{1}{\eta(T+1)}\left\|\theta^{(0)} - \theta^{\star}\right\|^2 + \frac{\bar{\sigma}^2}{n}\eta + \frac{36L\bar{\tau}^2}{p^2}\eta^2.
    \end{align*}
Fixing $\eta = \min\left\{\left(\frac{r_0}{b(T+1)}\right)^{\frac{1}{2}},    \left(\frac{r_0}{e(T+1)}\right)^{\frac{1}{3}} ,      \frac{1}{d} \right\}$ with $b=\frac{\bar{\sigma}^2}{n}$, $e=\frac{36L\bar{\tau}^2}{p^2}$, $d=\frac{8L}{p}$ and $r_0=\|\theta^{(0)}-\theta^\star\|_2^2$, then applying Lemma \ref{lemma:stepsize} that is recalled after, we obtain the final result.
\end{proof}

\begin{lemma}[Convergence rate with $T$ - Non convex case]
  \label{lemma:rate_TNC}  Consider the
  setting of Theorem \ref{thme:main} in the non-convex case. There exists a
  constant stepsize
  $\eta\leq \eta_{\max}=\frac{p}{8L}$ such that
      \begin{equation}
          \frac{1}{T+1}\sum_{t=0}^T\EE \Big\| \nabla f(\bar{\theta}^{(t)})  \Big\|_2^2 \leq 2\left(\frac{4bf_0}{T+1}\right)^{\frac{1}{2}} + 2e^{\frac{1}{3}}\left(\frac{4f_0}{T+1}\right)^{\frac{2}{3}}+\frac{4df_0}{T+1},
      \end{equation}
      where $b=\frac{2L\bar{\sigma}^2}{n}$, $e=\frac{96L^2\bar{\tau}^2}{p^2}$, $d=\frac{8L}{p}$ and $f_0=f(\theta^{(0)}) - f^\star$.
  \end{lemma}

\begin{proof}
  Similarly to Lemma \ref{lemma:rate_T} for the convex case, we can use the descent Lemma \ref{lemma:descentNC} and obtain
  \begin{align*}
     \EE \Big\| \nabla f(\bar{\theta}^{(t)})  \Big\|_2^2 \leq \frac{4}{\eta}\Big(\EE f_t - \EE f_{t+1} + \frac{L^2\eta}{n} \EE \left\|\Theta^{(t)}-\overline{\Theta}^{(t)}\right\|_F^2 + \frac{L\bar{\sigma}^2}{2n}\eta^2 \Big), 
  \end{align*}
where for all $t\geq 0$, $f_t\triangleq f(\bar{\theta}^{(t)}) - f^\star$. Then summing up and dividing by $T+1$ we get:
 \begin{align*}
  \frac{1}{T+1}\sum_{t=0}^T \EE \Big\| \nabla f(\bar{\theta}^{(t)})  \Big\|_2^2 & \leq \frac{4f_0}{\eta(T+1)} + \frac{4L^2}{n}\frac{1}{T+1}\sum_{t=0}^T \EE \left\|\Theta^{(t)}-\overline{\Theta}^{(t)}\right\|_F^2 + \frac{2L\bar{\sigma}^2}{n}\eta \\
  & \overset{\eqref{eq:cons_recursion}}{\leq}\frac{4f_0}{\eta(T+1)} + \frac{4L^2}{n}  \frac{24\eta^2n\bar{\tau}^2}{p^2} + \frac{2L\bar{\sigma}^2}{n}\eta \\
  & = \frac{4f_0}{\eta(T+1)} + \frac{96L^2\bar{\tau}^2}{p^2}\eta^2 + \frac{2L
  \bar{\sigma}^2}{n}\eta ~.
 \end{align*}

 Fixing $\eta = \min\left\{\left(\frac{4f_0}{b(T+1)}\right)^{\frac{1}{2}},    \left(\frac{4f_0}{e(T+1)}\right)^{\frac{1}{3}} ,      \frac{1}{d} \right\}$ with $b=\frac{2L\bar{\sigma}^2}{n}$, $e=\frac{96L^2\bar{\tau}^2}{p^2}$, $d=\frac{8L}{p}$ and $f_0=f(\theta^{(0)}) - f^\star$, we can apply Lemma \ref{lemma:stepsize} and obtain the final result.
\end{proof}

\begin{lemma}[Tuning stepsize \citep{koloskova20}] \label{lemma:stepsize} For
any parameter $r_0,b,e,d\geq0$, $T\in\mathbb{N}$, we can fix 
    \begin{equation*}
        \eta = \min\left\{\left(\frac{r_0}{b(T+1)}\right)^{\frac{1}{2}},    \left(\frac{r_0}{e(T+1)}\right)^{\frac{1}{3}} ,      \frac{1}{d} \right\} \leq \frac{1}{d},
    \end{equation*}
    and get
    \begin{equation*}
        \frac{r_0}{\eta (T+1)} +b\eta +e\eta^2 \leq 2\left(\frac{br_0}{T+1}\right)^{\frac{1}{2}} + 2e^{\frac{1}{3}}\left(\frac{r_0}{T+1}\right)^{\frac{2}{3}}+\frac{dr_0}{T+1}.
    \end{equation*}
\end{lemma}

\begin{proof}
The proof of this lemma can be found in the supplementary materials of \cite{koloskova20} (Lemma 15).
\end{proof}

\section{Additional Results and Proofs}
\label{app:other-proofs}

\textbf{Proposition \ref{prop:H_bound}.}
\emph{        Let Assumptions \ref{ass:variance}-\ref{ass:rho} and \ref{ass:heterogeneityA} to be verified. Then Assumption~\ref{ass:new} is satisfied with
$\bar{\tau}^2 = (1-p)\left(\bar{\zeta}^2 + \bar{\sigma}^2\right)$,
where $\bar{\sigma}^2\triangleq\frac{1}{n}\sum_i\sigma_i^2$.
}

\begin{proof} Denoting $\nabla F(\theta) = \left(\nabla F_1(\theta,Z_1), \ldots, \nabla F_n(\theta,Z_n)\right)\in \RR^{d\times n}$, and using the relation
\begin{equation}
\label{eq:basic_bias_variance_formula}
\EE\|Y\|_2^2 = \|\EE Y\|_2^2 + \EE \|Y - \EE
Y\|_2^2,
\end{equation}
we
have:
\begin{align*}
& H^{(t)} = \frac{1}{n}\EE\|\nabla F(\theta)W^{(t)} - \overline{\nabla F(\theta)}\|_F^2 \\
& \overset{(A.\ref{ass:rho})}{\leq} \frac{1-p}{n} \EE\|\nabla F(\theta) - \overline{\nabla F(\theta)}\|_F^2 \\
& =  \frac{1-p}{n} \sum_{i=1}^n \EE\Big\|\nabla F_i(\theta,Z_i) - \frac{1}
{n}\sum_{j=1}^n\nabla F_j(\theta,Z_j)\Big\|_2^2 \\
& \overset{\eqref{eq:basic_bias_variance_formula}}{=} \frac{1-p}
{n}\sum_{i=1}^n \Bigg(
\Big\|\nabla f_i(\theta) - \frac{1}{n}\sum_{j=1}^n\nabla f_j(\theta)\Big\|_2^2
+ \EE
\Big\|\sum_{j=1}^n\big(\indic_{\{j=i\}}-\frac{1}{n})\big ( \nabla F_j
(\theta,Z_j) -
\nabla f_j(\theta) ) \Big\|_2^2\Bigg) \\
& \overset{(A.\ref{ass:heterogeneityA})}{\leq} (1-p)\Bigg(\bar{\zeta}^2 + 
\frac{1}{n}\sum_{i=1}^n \EE \Big\|\sum_{j=1}^n\big(\indic_{\{j=i\}}-\frac{1}
{n}\big)
( \nabla F_j(\theta,Z_j) - \nabla f_j(\theta) ) \Big\|_2^2\Bigg).
\end{align*}

Since all terms $j$ in the norm are independent and with expectation $0$, the expectation of the sum is equal to the sum of expectations and 
\begin{align*}
    H^{(t)} & \leq (1-p)\left(\bar{\zeta}^2 + \frac{1}{n}\sum_{i=1}^n\sum_
    {j=1}^n \big(\indic_{\{j=i\}}-\frac{1}{n}\big)^2 \EE \left\| \nabla F_j
    (\theta,Z_j) - \nabla f_j(\theta)  \right\|_2^2\right) \\
    & = (1-p)\Bigg(\bar{\zeta}^2 + \frac{1}{n}\sum_{j=1}^n\EE \left\| \nabla
    F_j(\theta,Z_j) - \nabla f_j(\theta)  \right\|_2^2 \underbrace{\sum_
    {i=1}^n \big(\indic_{\{j=i\}}-\frac{1}{n}\big)^2}_{= \frac{n-1}{n}} \Bigg)
    \\
    & \overset{(A.\ref{ass:variance})}{\leq} (1-p)\left(\bar{\zeta}^2 + \frac{n-1}{n}\bar{\sigma}^2\right) \leq  (1-p)\left(\bar{\zeta}^2 + \bar{\sigma}^2\right),
\end{align*}
which concludes the proof.
\end{proof}

\vspace{1.5cm}

\textbf{Proposition \ref{prop:upper-bound-H}.}
\emph{
    \emph{(Bounded neighborhood heterogeneity under
    label skew)}  Consider the statistical framework defined above and assume there exists $B>0$ such that
    $\forall k=1,\ldots,K$ and $\forall\theta \in\RR^d$, $\|\EE_X[\nabla F(\theta ; X,Y)|Y=k] - \frac{1}{K}\sum_{k^\prime = 1}^K \EE_X[\nabla F(\theta ; X,Y)|Y=k^\prime]\|_2^2\leq B$.
 Then, denoting $\pi_{jk}\triangleq P_j(Y=k)$, Assumption \ref{ass:new} is satisfied with:
    \begin{align*}
       \bar{\tau}^2 =  \frac{KB}{n}\sum_{k=1}^K\sum_{i=1}^n  \Big(\sum_{j=1}^n  W_
        {ij}\pi_{jk}  - \frac{1}{n}\sum_{j=1}^n \pi_{jk}\Big)^2 + \frac{\sigma_{\max}^2}{n}\|W-\frac{1}{n}\1\1^\transpose\|_F^2~.
    \end{align*}
}

\begin{proof}

    First, observe that the local objective functions can be re-written

    \begin{align*}
        f_j(\theta) & = \EE_{(X,Y)\sim \calD_j}[F(\theta;X,Y)] \\
        & = \sum_{k=1}^K P_j(Y=k)\EE_X[F(\theta; X,Y)| Y= k] \\
        & = \sum_{k=1}^K \pi_{jk} \EE_X[F(\theta; X,Y)| Y= k] ~.
    \end{align*}

    From \eqref{eq:H_bias_variance}, we have the bias-variance decomposition

    \begin{align*}
        H(\theta) & \leq \frac{1}{n}\sum_{i=1}^n\left\|\sum_{j=1}^nW_{ij}\nabla f_j(\theta) - \nabla f(\theta)\right\|_2^2 + \frac{\sigma_{\max}^2}{n}\|W-\frac{1}{n}\1\1^\transpose\|_F^2 \\
        & =   \frac{1}{n}\sum_{i=1}^n\left\|\sum_{j=1}^n(W_{ij} - \frac{1}{n})  \nabla f_j(\theta)\right\|_2^2 + \frac{\sigma_{\max}^2}{n}\|W-\frac{1}{n}\1\1^\transpose\|_F^2 \\
        & =   \frac{1}{n}\sum_{i=1}^n\underbrace{\left\|\sum_{j=1}^n(W_{ij} - 
        \frac{1}{n})  \sum_{k=1}^K \pi_{jk} \EE_X[\nabla F(\theta; X,Y)| Y= k]\right\|_2^2}_{T_4} + \frac{\sigma_{\max}^2}{n}\|W-\frac{1}{n}\1\1^\transpose\|_F^2 ~.
    \end{align*}

    Then, observing that $\sum_{j=1}^n(W_{ij} - \frac{1}{n}) = 0$ and $\sum_{k=1}^K \pi_{jk} = 1$ imply 

    \begin{equation*}
        \sum_{j=1}^n(W_{ij} - \frac{1}{n})  \sum_{k=1}^K \pi_{jk} \frac{1}{K}\sum_{k^\prime =1}^n\EE_X[\nabla F(\theta; X,Y)| Y= k^\prime] = \textbf{0} ~,
    \end{equation*}

    we can add this in the norm of the term $T_4$ defined above and get
    \begin{align*}
        & T_4 = \Bigg\|\sum_{j=1}^n(W_{ij} - \frac{1}{n})  \sum_{k=1}^K \pi_{jk} \Big(\EE_X[\nabla F(\theta; X,Y)| Y= k] - \frac{1}{K}\sum_{k^\prime = 1}^K \EE_X[\nabla F(\theta ; X,Y)|Y=k^\prime]\Big)\Bigg\|_2^2 \\
            & = \Bigg\| \sum_{k=1}^K  \Big(\EE_X[\nabla F(\theta; X,Y)| Y= k] - \frac{1}{K}\sum_{k^\prime = 1}^K \EE_X[\nabla F(\theta ; X,Y)|Y=k^\prime]\Big) \sum_{j=1}^n(W_{ij} - \frac{1}{n}) \pi_{jk}\Bigg\|_2^2 \\
            &  \overset{\eqref{eq:norm_sum}}{\leq}  K \sum_{k=1}^K \Bigg\|  \Big(\EE_X[\nabla F(\theta; X,Y)| Y= k] - \frac{1}{K}\sum_{k^\prime = 1}^K \EE_X[\nabla F(\theta ; X,Y)|Y=k^\prime]\Big) \sum_{j=1}^n(W_{ij} - \frac{1}{n}) \pi_{jk}\Bigg\|_2^2 \\
            & = K\sum_{k=1}^K \underbrace{\left\|  \EE_X[\nabla F(\theta; X,Y)| Y= k] - \frac{1}{K}\sum_{k^\prime = 1}^K \EE_X[\nabla F(\theta ; X,Y)|Y=k^\prime] \right\|_2^2}_{\leq B} \\
            & \pushright{\times \Big(\sum_{j=1}^n(W_{ij} - \frac{1}{n}) \pi_{jk}\Big)^2 }\\
            & \leq KB \sum_{k=1}^K \Bigg(\sum_{j=1}^n  W_{ij}\pi_{jk} - 
            \frac{1}{n}\sum_{j=1}^n \pi_{jk}\Bigg)^2~.
    \end{align*}

    Finally, plugging this into the upper-bound on $H(\theta)$ found above, we get the
    final result.
\end{proof}

\vspace{1.5cm}

\textbf{Theorem \ref{coro:bound}} \emph{Consider the statistical setup presented in Section~\ref{sec:framework_label_skew} and let $\{\widehat{W}^{(l)}\}^L_{l = 1}$
be the sequence of mixing matrices generated by Algorithm~\ref{alg:FW-W}. Then, at any iteration $l=1,\ldots,L$, we have:
\begin{equation*}
    \textstyle
    g(\widehat{W}^{(l)}) \leq \frac{16}{l+2}\big(\lambda + \frac{1}
    {n}\big\|\sum_{k=1}^K(\Pi_{:,k}-\overline{\Pi_{:,k}}\1)\cdot \Pi_
    {:,k}^\transpose \big\|^\star_{2} \big)~,
\end{equation*}
where $\left\|\cdot\right\|^\star_{2}$ stands for the nuclear norm, i.e., the sum of singular values.
 Bounding the second term in the parenthesis, we can obtain the looser bound
    \begin{equation*}
        \textstyle
        g(\widehat{W}^{(l)}) \leq \frac{16}{l+2}\left(\lambda + 1 \right)~.
    \end{equation*}
    Furthermore, we have $d^{\text{in}}_{\text{max}}(\widehat{W}^{(l)})\leq l$ and $d^{\text{out}}_{\text{max}}(\widehat{W}^{(l)})\leq l$,
resulting in a per-iteration complexity bounded by $l$.
}

\begin{proof}
    The proof of this theorem is directly derived from Theorem~\ref{thme:FW-conv} given below, applied with the parameters of our problem.
    To prove the first inequality, we first need to find a bound on the diameter of the set of doubly stochastic matrices, denoted $\text{diam}_{\|\cdot\|}(\calS)$, for a certain (matrix) norm $\|\cdot\|$. We fix this norm to be the operator norm induced by the $\ell_2$-norm, denoted $\left\|\cdot\right\|_2$, which is simply the maximum singular value of the matrix.

    For all $W,P\in \calS$, we have
    \begin{align*}
        \left\|W - P\right\|_{2} & \leq  \left\|W\right\|_2 + \left\|P\right\|_2 \\
        & = 1 + 1 = 2~,
    \end{align*}
    which comes from the fact that $W$ and $P$ are doubly stochastic, i.e.,
    their largest eigenvalue is $1$. This shows that $\text{diam}_{\|\cdot\|_2}(\calS)\leq 2$.

    Let us now find the Lipschitz constant associated to the gradient of the
    objective:
    \begin{equation*}
        \nabla g(W) = \frac{2}{n}\sum_{k=1}^K(W\Pi_{:,k}-\overline{\Pi_{:,k}}\1)\cdot \Pi_{:,k}^\transpose + \frac{2}{n}\lambda \left(W - \frac{\1\1^\transpose}{n}\right)~.
    \end{equation*}

    Recall that the dual norm $\left\|\cdot\right\|^\star_1$ of $\left\|\cdot\right\|_1$ is the nuclear norm, i.e., the sum of the singular values.

    For any $W,P\in\calS$, we have
    \begin{align*}
        \left\|\nabla g(W) - \nabla g(P)\right\|^\star_{2} & = \frac{2}{n}\left\|(W-P)\left(\lambda I + \sum_{k=1}^K\Pi_{:,k}\Pi_{:,k}^\transpose \right)\right\|^\star_{2} \\
        & \leq \frac{2}{n}\left\|\lambda (W-P)I \right\|^\star_{2} + \frac{2}{n}\left\| (W-P) \sum_{k=1}^K\Pi_{:,k}\Pi_{:,k}^\transpose \right\|^\star_{2} \\
        & \leq \frac{2\lambda}{n}\left\|W-P\right\|_{2}\left\|I\right\|^\star_{2} + \frac{2}{n} \left\| (W-P) \sum_{k=1}^K\Pi_{:,k}\Pi_{:,k}^\transpose \right\|^\star_{2}~,
    \end{align*}
    where the last inequality is obtained using the fact that for any real
    matrices $A$ and $B$, $\left\|AB\right\|^\star \leq \|A^\transpose\|\|B\|^\star$. 
    
    Before bounding the second term, we must observe that because $W$ and $P$ are doubly stochastic, $(W-P)\1 = 0$ and therefore, for any matrix $A\in \RR^{n\times n}$, $(W-P)A = (W-P)(A - \frac{\1\1^\transpose}{n} A)$.

    Now, the second term can be re-written and bounded as follows:
    \begin{align*}
        \frac{2}{n}\left\| (W-P) \sum_{k=1}^K\Pi_{:,k}\Pi_{:,k}^\transpose \right\|^\star_{2} & =  \frac{2}{n}\left\| (W-P) \left(\sum_{k=1}^K\Pi_{:,k}\Pi_{:,k}^\transpose -  \frac{\1\1^\transpose}{n}\sum_{k=1}^K\Pi_{:,k}\Pi_{:,k}^\transpose \right) \right\|^\star_{2} \\
        & \leq \frac{2}{n}\left\|W-P\right\|_2 \left\|\sum_{k=1}^K\Pi_{:,k}\Pi_{:,k}^\transpose -  \frac{\1\1^\transpose}{n}\sum_{k=1}^K\Pi_{:,k}\Pi_{:,k}^\transpose\right\|_2^\star \\
        & = \frac{2}{n}\left\|W-P\right\|_2 \left\|\sum_{k=1}^K(\Pi_{:,k}-
        \overline{\Pi_{:,k}}\1)\cdot \Pi_{:,k}^\transpose\right\|_2^\star ~.
    \end{align*}

    Plugging the previous result into the bound obtained above, and since
    $\left\|I\right\|^\star_{2} = n$, we get
    \begin{equation*}
        \left\|\nabla g(W) - \nabla g(P)\right\|^\star_{2}  \leq 2 \left( \lambda + \frac{1}{n}\left\|\sum_{k=1}^K(\Pi_{:,k}-\overline{\Pi_{:,k}}\1)\cdot \Pi_{:,k}^\transpose\right\|_2^\star \right) \left\|W-P\right\|_2~.
    \end{equation*}

    We can now apply Theorem~\ref{thme:FW-conv} with the found
    Lipschitz constant and diameter, which gives:
    \begin{equation*}
        g(\widehat{W}^{(l)}) - g(W^\star) \leq \frac{16}{l+2}\left(\lambda + \frac{1}{n}\left\|\sum_{k=1}^K(\Pi_{:,k}-\overline{\Pi_{:,k}}\1)\cdot \Pi_{:,k}^\transpose \right\|^\star_{2} \right)~,
    \end{equation*}
    where $W^\star$ is the optimal solution of the problem. Since we known
    that $W^\star = \frac{\1\1^\transpose}{n}$ with $g(W^\star)= 0$, we obtain
    the first inequality in Theorem~\ref{coro:bound}.

    To prove the second inequality, it suffices to show that $\left\|\sum_
    {k=1}^K(\Pi_{:,k}-\overline{\Pi_{:,k}}\1)\cdot \Pi_{:,k}^\transpose
    \right\|^\star_{2}  \leq n$. We have:
    \begin{align*}
        \left\|\sum_{k=1}^K(\Pi_{:,k}-\overline{\Pi_{:,k}}\1)\cdot \Pi_{:,k}^\transpose \right\|^\star_{2} & = \left\|\left(I - \frac{\1\1^\transpose}{n}\right)\sum_{k=1}^K\Pi_{:,k}\Pi_{:,k}^\transpose \right\|^\star_{2} \\
        & \leq \left\|I - \frac{\1\1^\transpose}{n}\right\|_2  \left\|\sum_{k=1}^K\Pi_{:,k}\Pi_{:,k}^\transpose \right\|^\star_{2} \\
        & = \left\|\sum_{k=1}^K\Pi_{:,k}\Pi_{:,k}^\transpose \right\|^\star_{2} \\
        & \leq \sum_{k=1}^K \left\|\Pi_{:,k}\Pi_{:,k}^\transpose
        \right\|^\star_{2}~. 
    \end{align*}

    Because for any $k=1,\ldots,K$, $\Pi_{:,k}\Pi_{:,k}^\transpose$ is a
    rank-$1$ matrix, its unique eigenvalue is $\Pi_{:,k}^\transpose\Pi_{:,k}$ and therefore

    \begin{align*}
        \left\|\sum_{k=1}^K(\Pi_{:,k}-\overline{\Pi_{:,k}}\1)\cdot \Pi_{:,k}^\transpose \right\|^\star_{2} & \leq \sum_{k=1}^K \left\|\Pi_{:,k}\Pi_{:,k}^\transpose \right\|^\star_{2} \\
        & = \sum_{k=1}^K \Pi_{:,k}^\transpose\Pi_{:,k} \\
        & = \sum_{k=1}^K \sum_{i=1}^n \pi_{ik}^2 \\
        & \overset{\text{Holder}}{\leq} \sum_{i=1}^n \max_{k}\{\pi_{ik}\}\underbrace{\sum_{k=1}^K \pi_{ik}}_{= 1} \\
        & \leq \sum_{i=1}^n 1 = n~,
    \end{align*}
    which concludes the proof of the second inequality in Theorem~\ref{coro:bound}.

    The last statement of the theorem follows directly from the structure of permutation matrices and the greedy nature of the algorithm.
\end{proof}

\vspace{1.5cm}

\begin{theorem}\label{thme:FW-conv}\emph{(Frank-Wolfe Convergence  \citep{jaggi2013revisiting,bubeck2014convex})}
    Let the gradient of the objective function $g:x\rightarrow g(x)$ be $L$-smooth with respect to a norm $\|\cdot\|$ and its dual norm $\|\cdot\|^{\star}$:
    \begin{equation*}
        \|\nabla g(x) - \nabla g(y)\|^{\star} \leq L \|x - y\|~.
    \end{equation*}
    
    If $g$ is minimized over $\calS$ using Frank-Wolfe algorithm, then for each $l\geq 1$, the iterates $x^{(l)}$ satisfy
    \begin{equation*}
        g(x^{(l)}) - g(x^\star) \leq \frac{2L\text{diam}_{\|\cdot\|}(\calS)^2}{l + 2}~,
    \end{equation*} 
    where $x^\star \in \calS$ is an optimal solution of the problem and $\text{diam}_{\|\cdot\|}(\calS)$ stands for the diameter of $\calS$ with respect to the norm $\|\cdot\|$.
    
\end{theorem}

\begin{proof}
    The proof of this theorem is a direct combination of Theorem $1$ and Lemma $7$ in \cite{jaggi2013revisiting}, both proved in the paper.
\end{proof}


\begin{proposition}\emph{(Relation between $p$ and $\|W-\frac{\1\1^\transpose}{n}\|_F^2$)} 
    \label{prop:relation-p-frob}
    Let $W$ be a mixing matrix satisfying Assumption \ref{ass:rho}. Then,
\begin{equation*}
    (1-p)\leq \left\|W-\frac{\1\1^\transpose}{n}\right\|_F^2 \leq (n-1)(1-p)~.
\end{equation*}
    
\end{proposition}

\begin{proof}
    The upper-bound is a direct application of Assumption \ref{ass:rho} with $M=I$, the identity matrix of size $n$: 
    \begin{align*}
        \left\|W^\transpose-\frac{\1\1^\transpose}{n}\right\|_F^2 = \left\|IW^\transpose-I\frac{\1\1^\transpose}{n}\right\|_F^2 \overset{(A.\ref{ass:rho})}{\leq} (1-p)\left\|I-\frac{\1\1^\transpose}{n}\right\|_F^2 = (1-p)(n-1)~.
    \end{align*}

    To show the lower-bound, denote by $s_1(M),\ldots,s_n(M)$ the (decreasing) singular values of any square matrix $M\in \RR^{n\times n}$. Denote similarly $\lambda_1(M),\ldots,\lambda_n(M)$ the eigenvalues of any \emph{symmetric} square matrix $M\in \RR^{n\times n}$.

    \begin{align*}
        \left\|W-\frac{\1\1^\transpose}{n}\right\|_F^2 = \sum_{i=1}^ns_i^2\left(W-\frac{\1\1^\transpose}{n}\right) & \geq s_1^2\left(W-\frac{\1\1^\transpose}{n}\right) \\
        & = \lambda_1\left((W-\frac{\1\1^\transpose}{n})^\transpose(W-\frac{\1\1^\transpose}{n})\right) \\
        & = \lambda_1\left(W^\transpose W-\frac{\1\1^\transpose}{n}\right) \\
        & = \lambda_2(W^\transpose W) \geq 1-p~.
    \end{align*}

        The last \emph{equality} is obtained by noticing that $W^\transpose W$ is a symmetric doubly stochastic matrix. 
        It therefore admits an eigenvalue decomposition where the largest eigenvalue $1$ is associated with the eigenvector $\frac{1}{\sqrt{n}}\1$. This makes $W^\transpose W-\frac{\1\1^\transpose}{n}$ having the eigenvalue $0$ associated to the vector $\frac{1}{\sqrt{n}}\1$ and the largest eigenvalue of $W^\transpose W-\frac{\1\1^\transpose}{n}$ becomes the second-largest eigenvalue of $W^\transpose W$. 
        The final \emph{inequality} comes from the fact that Assumption \ref{ass:rho} is always true with $p = 1-\lambda_2(W^\transpose W)$ which implies that the best $p$ satisfying Assumption \ref{ass:rho} in necessarily greater or equal to $1-\lambda_2(W^\transpose W)$.
\end{proof}


\subsection{Extension to Random Mixing Matrices}
\label{sec:random-mixing}

As mentioned in Section~\ref{sec:background}, all our theoretical results can easily be extended to random mixing matrices. In that framework, at each iteration $t$ of the D-SGD algorithm, the matrix $W^{(t)}$ is sampled from a doubly stochastic matrix distribution denoted $\mathcal{W}^{(t)}$, independent of the iterates of parameters $\theta^{(t)}$, 
and possibly time-varying.

To obtain the convergence result, we slightly modify Assumption~\ref{ass:rho} and Assumption~\ref{ass:new} by adding an expectation with respect to $W$ in front of the equations. For instance, Assumption~\ref{ass:rho} becomes $\EE_{W \sim \mathcal{W}}\|MW^\transpose-\overbar{M}\|_F^2\leq (1-p)\|M-\overbar{M}\|_F^2$. Then, the statement of Theorem \ref{thme:main} is also slightly modified by assuming that it is the distributions $\mathcal{W}^{(0)}, \ldots, \mathcal{W}^{(T-1)}$ that must now respect Assumptions \ref{ass:rho} and \ref{ass:new}. 

By appropriately conditioning with respect to the random mixing matrices or with respect to the iterates, the proof of the theorem remains the same.


\subsection{Closed-Form for the Line-Search}
\label{sec:line-search}

In this section, we give the closed-form solution of the line-search problem found in the Frank-Wolfe algorithm \ref{alg:FW-W}. Recall that we seek to solve:

\begin{equation*}
    \gamma^\star = \underset{\gamma \in [0,1]}{\arg \min} \hspace{0.5em} \left\{ \tilde{g}(\gamma) \triangleq g\left((1-\gamma) W + \gamma P\right)\right\}~,
\end{equation*}
with

\begin{equation*}
    g(W)  =  \frac{1}{n}\Big\|W \Pi -  \frac{\1\1^\transpose}{n}\Pi\Big\|_F^2 + \frac{\lambda}{n} \Big\|W -  \frac{\1\1^\transpose}{n}\Big\|_F^2~.
\end{equation*}

The function $g$ being quadratic, the objective $\tilde{g}(\gamma)$ is also quadratic with respect to $\gamma$. Hence, it suffices to put the derivative $\tilde{g}^\prime$ of $\tilde{g}$ equal to $0$, and we get the closed-form solution:

\begin{equation*}
    \gamma^\star = \frac{\sum_{k=1}^K(\overline{\Pi_{:,k}}\1  -  W\Pi_{:,k})^\transpose (P-W)\cdot \Pi_{:,k}-\lambda \cdot \text{tr}\left( \left(W - \frac{\1\1^\transpose}{n}\right)^\transpose (P-W)\right)}{\left\|(P-W)\Pi\right\|_F^2 + \lambda \|P-W\|_F^2}~.
\end{equation*}

\section{Additional Experiments}
\label{sec:additional-experiments}

In this section, we provide additional details on our experimental setup, as well as additional results to complement the main results in the paper.

\subsection{Detailed Experimental Setup}
\label{appendix:experimental-setup}

Our main goal is to provide a fair comparison of the convergence speed across
different topologies in order to show the benefits of the principled approach to
topology learning provided by STL-FW. We essentially follow the experimental
setup in \cite{bellet2021d}, which we recall below.

In our study, we focus our investigation
on the convergence speed, rather than
the final accuracy after a fixed number of iterations.
Indeed, depending on when training is stopped, the relative 
difference in final accuracy across different algorithms may vary
significantly and lead to different conclusions. Instead of relying on
somewhat arbitrary stopping points, we show the convergence curves of
generalization performance (i.e., the accuracy on the test set throughout
training), up to a point where it is clear that the different approaches have
converged, will not make significantly more progress, or behave essentially
the same. 

\paragraph{Datasets.} We experiment with two datasets: MNIST~
\citep{deng2012mnist} and
CIFAR10~\citep{krizhevsky2009learning}, which both have $K=10$ classes.
For MNIST,  we use 50k and 10k examples from the original
set for training and testing respectively. 
For CIFAR10, we used 50k images 
of the original training set for training and 10k examples 
 of the test set for measuring prediction accuracy. 

For both MNIST and CIFAR10, we use the heterogeneous data partitioning scheme
proposed by \cite{mcmahan2017communication} 
in their seminal FL work: 
we sort all training examples by class, then split the list into shards of
equal size, and randomly assign two shards to each node. When the number of
examples of one class does not divide evenly in shards, as is the case for MNIST, some shards may have examples of more than one class and therefore nodes may have examples
of up to 4 classes. However, most nodes will have examples of 2 classes.  

\paragraph{Models.}
We
use a logistic regression classifier for MNIST, which
provides up to 92.5\% accuracy in the centralized setting.
For CIFAR10, we use a Group-Normalized variant of LeNet~\citep{quagmire}, a
deep convolutional network which achieves an accuracy of $74.15\%$ in the
centralized setting.
These models are thus reasonably accurate (which is sufficient to
study the effect of the topology) while being sufficiently fast to train in a
fully decentralized setting and simple enough to configure and analyze.
Regarding hyper-parameters, we use the learning rate and
mini-batch size found in \cite{bellet2021d} after cross-validation for $n = 100$ nodes, respectively $0.1$ and $128$ for
MNIST and $0.002$ and $20$ for CIFAR10.

\paragraph{Metrics.}
We evaluate a network of $n = 100$ nodes,
creating multiple models
in memory and simulating the exchange of messages between nodes.
To ignore the impact of distributed execution strategies and system
optimization techniques, we report the test accuracy of all nodes (min, max,
average) as a function of the number of times each example of the dataset has
been sampled by a node, i.e. an \textit{epoch}. This is equivalent to the classic 
case of a single node sampling the full distribution. All our results were obtained
on a custom version of the \textit{non-iid topology simulator} made available
online by the authors of \cite{bellet2021d},\footnote{\url{https://gitlab.epfl.ch/sacs/distributed-ml/non-iid-topology-simulator}} which provides deterministic and
fully replicable experiments on top of Pytorch and ensures all topologies were
used in the same algorithm implementation and used exactly the same
inputs.

\paragraph{Baselines} We compare our results against an ideal
baseline:
a fully-connected network topology with the same number of nodes. All other
things being equal, any other topology using less edges will converge
at the same speed or slower: \textit{this is therefore the most difficult and general baseline to compare against}.
This baseline is also essentially equivalent to a centralized (single) IID node using a batch size
$n$ times bigger, where $n$ is the number of nodes.  Both a fully-connected network and a single IID node
 effectively optimize a single model and sample
uniformly from the global distribution: both thus remove entirely the
effect of label distribution skew and of the network topology on the
optimization. In practice, we prefer a
fully-connected network because it
 converges slightly faster and obtains slightly 
better final accuracy than a single node sampling randomly from the global
distribution.

We also provide comparisons against popular sparse topologies, such as random
graphs and exponential graphs \citep{ying2021exponential}.
For the random graph, we use a similar 
number of edges ($d_{\max}$) per node to determine
whether a simple sparse topology could work equally well. 
For the exponential graph, we follow 
the deterministic construction of~\cite{ying2021exponential} and consider
edges to be undirected, resulting in $d_{max} = 14$ for $n=100$.

We finally compare against D-Cliques~\citep{bellet2021d}, the only competitor
which takes into account the data heterogeneity in the choice of topology.
D-Cliques constructs a topology around sparsely inter-connected cliques
 such that the union of local datasets within a clique is representative of
 the global distribution, i.e. it minimizes the first term in our
 objective function (Eq.~\ref{eq:obj_FW}) within each clique.
 

\subsection{Statistics of the Used Topologies}
\label{sec:statistics_topologies}

In this section, we provide tables containing important statistics about
the topologies used in the experiments. In each table, a row
corresponds to a specific topology having at most $d_{\max}$ in and
out-neighbors per node.
The columns are as follows:

\begin{itemize}
    \item \emph{In-degree} (respectively \emph{Out-degree}): average and
    standard deviation of the number of incoming (respectively outgoing) edges
    per node.
    \item \emph{Classes in neighborhood}: average and standard
    deviation of the number of different classes in the direct
    neighborhood of a node. Recall that each node individually observes
    examples only from a subset of the $10$ different classes (1 for the
    synthetic dataset, 2 for MNIST, 2-4 for CIFAR10).
    \item \emph{Bias}: average and standard deviation of
    $(\sum_{j=1}^n   W_{ij}\pi_{jk}  - \frac{1}{n}\sum_{j=1}^n\pi_{jk})^2$ across
    each node $i$.
    In other words, it measures the neighborhood heterogeneity in terms of
    class proportions, which (up to a constant factor) corresponds to the bias
    term in \eqref{eq:prop2}.
    According to our
    theory, the smaller the bias term, the better the topology.
    \item $1-p$: the mixing parameter of the topology
    (see Assumption \ref{ass:rho}). Recall that for a topology $W$,
    $1-p = \lambda_2(W^\transpose W)$. According to our
    theory \citep[and prior work, see e.g.,][]{koloskova20}, the smaller
    $1-p$, the better the topology.
\end{itemize}  

Interestingly, all tables show that our algorithm STL-FW outputs topologies
that are $d_{\max}$-regular. Therefore, the communication burden is the
same for all nodes. This is a desirable property for scalability, that the
star topology induced by server-based federated learning does not satisfy.

\begin{table}[t]
    \centering
    \begin{tabular}{c|c|c|c|c|c|c}
        \hline
        & Topology & In-degree & Out-degree & Classes in neighborhood & Bias & $1-p$ \\
        \hline
        \hline
        \multirow{2}{*}{$d_{\max}=3$} & STL-FW (ours) & $3.0 \pm 0.0$ & $3.0
        \pm 0.0$ & $4.0 \pm 0.0$ &  $0.15 \pm 0.0$  & $0.85$ \\
        & Random $d$-regular & $3.0 \pm 0.0$ & $3.0 \pm 0.0$ & $3.88 \pm 0.38$
        & $0.28 \pm 0.1$  & $0.89$ \\
        \hline
        \multirow{2}{*}{$d_{\max}=9$} & STL-FW (ours) & $9.0 \pm 0.0$ & $9.0 \pm 0.0$ & $10.0 \pm 0.0$ & $0.0 \pm 0.0$  & $0.41$ \\
        & Random $d$-regular & $9.0 \pm 0.0$ & $9.0 \pm 0.0$ & $9.65 \pm 0.65$
        & $0.09 \pm 0.04$  & $0.39$ \\
        \hline
    \end{tabular}
    \caption{Statistics of the topologies used in the synthetic data experiments of Section \ref{sec:synthetic}.}   
    \label{tab:synth}     
\end{table}

Table \ref{tab:synth} provides the statistics of the topologies used in the
synthetic data experiment. We observe that the mixing parameter $p$ are
similar for both our topologies (STL-FW) and the random $d$-regular graph.
However, STL-FW achieves much smaller bias, resulting in more classes being
represented in the neighborhood of each node. This explains the faster
convergence of D-SGD with our topology, in line with our theoretical results.

\begin{table}[t]
    \centering
    \begin{tabular}{c|c|c|c|c|c|c}
        \hline
        & Topology & In-degree & Out-degree & Classes in neighborhood & Bias & $1-p$ \\
        \hline
        \hline
        \multirow{2}{*}{$d_{\max}=2$} & STL-FW (ours) & $2.0 \pm 0.0$ & $2.0 \pm 0.0$  & $ 6.03 \pm 0.62$  & $0.08 \pm 0.02$  & $0.88$ \\
        & Random $d$-regular & $2.0 \pm 0.0$  & $2.0 \pm 0.0$  &$ 4.94 \pm 0.89$ & $0.14 \pm 0.06$ & $0.99$ \\
        \hline
        \multirow{3}{*}{$d_{\max}=5$} & STL-FW (ours)  & $5.0 \pm 0.0$   &  $5.0 \pm 0.0$   &  $9.99 \pm 0.1$ & $0.007 \pm 0.004$ & $0.55$ \\
        & Random $d$-regular &  $5.0 \pm 0.0$ & $5.0 \pm 0.0$ & $7.53 \pm 1.07$ & $0.07 \pm 0.03$ & $0.68 $\\
        & D-cliques &  $5.82 \pm 0.38$ & $5.82 \pm 0.38$  & $9.81 \pm 0.39$ & $0.02 \pm 0.01$ & $0.99$ \\
        \hline
        \multirow{4}{*}{$d_{\max} = 10$} & STL-FW (ours)   & $10.0 \pm 0.0$  &  $10.0 \pm 0.0$   &   $10.0 \pm 0.0$ & $0.001 \pm 0.001$ & $0.35$ \\
        & Random $d$-regular   &  $10.0 \pm 0.0$  &  $10.0 \pm 0.0$   &  $9.31 \pm 0.76$ & $0.03 \pm 0.02$ & $0.39$ \\
        & D-cliques & $9.9 \pm 0.3$  & $9.9 \pm 0.3$  &  $10.0 \pm 0.0$ & $0.005 \pm 0.002$ & $0.84$ \\
        & Exponential & $14.0 \pm 0.0$  &  $14.0 \pm 0.0$  &  $9.72 \pm 0.51$ & $ 0.02 \pm 0.01$ & $0.54$ \\
        \hline
    \end{tabular}
    \caption{Statistics of the topologies used on the MNIST experiments of
    Section \ref{sec:real-datasets}.} \label{tab:mnist}
\end{table}

\begin{table}[t]
    \centering
    \begin{tabular}{c|c|c|c|c|c|c}
        \hline
        & Topology & In-degree & Out-degree & Classes in neighborhood & Bias & $1-p$ \\
        \hline
        \hline
        \multirow{2}{*}{$d_{\max}=2$} & STL-FW (ours) & $2.0 \pm 0.0$ & $2.0 \pm 0.0$  & $ 5.79 \pm 0.45$  & $0.08 \pm 0.02$  & $0.99 $\\
        & Random $d$-regular & $2.0 \pm 0.0$  & $2.0 \pm 0.0$  & $4.86 \pm 0.81$  & $0.14 \pm 0.06$ & $0.99$ \\
        \hline
        \multirow{3}{*}{$d_{\max}=5$} & STL-FW (ours)  & $5.0 \pm 0.0$   &  $5.0 \pm 0.0$   &  $9.98 \pm 0.14$ & $0.008 \pm 0.005$ & $0.64$ \\
        & Random $d$-regular &  $5.0 \pm 0.0$ & $5.0 \pm 0.0$ & $7.4 \pm 0.97$ & $0.07 \pm 0.03$ & $0.68$ \\
        & D-cliques &  $5.82 \pm 0.38$ & $5.82 \pm 0.38$  &  $9.71 \pm 0.55$ & $0.022 \pm 0.012$ & $0.99$ \\
        \hline
        \multirow{4}{*}{$d_{\max} = 10$} & STL-FW (ours)   & $10.0 \pm 0.0$  &  $10.0 \pm 0.0 $  &   $10.0 \pm 0.0$ & $0.001 \pm 0.001$ & $0.45$ \\
        & Random $d$-regular   &  $10.0 \pm 0.0$  &  $10.0 \pm 0.0$   &  $9.26 \pm 0.82$ & $0.033 \pm 0.016$ & $0.39 $\\
        & D-cliques   & $9.9 \pm 0.3$  & $9.9 \pm 0.3$  &  $10.0 \pm 0.0$ & $0.004 \pm 0.002$ & $0.84$ \\
        & Exponential & $14.0 \pm 0.0$  &  $14.0 \pm 0.0$  &  $9.68 \pm 0.58$ &  $0.024 \pm 0.012$ & $0.54$ \\
        \hline
    \end{tabular}
    \caption{Statistics of the topologies used in the CIFAR10 experiments of Section \ref{sec:real-datasets}.}
    \label{tab:cifar}      
\end{table}

Table \ref{tab:mnist} (MNIST) and Table \ref{tab:cifar} (CIFAR10) provide the
statistics of the topologies used in the real data experiments. The same
conclusions made regarding the synthetic experiments hold here regarding the
comparison of STL-FW with the random $d$-regular graphs. D-Cliques 
\citep{bellet2021d}, the
only topology that is also constructed in a data-dependent fashion,
achieves rather low bias (albeit slightly larger than our STL-FW topology)
but has rather bad mixing properties (large $1-p$). This confirms our
claim that D-Cliques reduces the bias without ensuring good
mixing (due to the constrained arrangements of nodes in sparsely
interconnected cliques). This explains the superior performance of our
topology. Last but not least, looking at $d_{\max}=5$, we notice that
D-Cliques is unable to satisfy the constraints that the maximum degree should
not exceed $5$. This also illustrates the greater flexibility of STL-FW
when it comes to controlling the per-iteration communication
complexity.

\subsection{Impact of $\lambda$ in STL-FW}


Figure~\ref{fig:lambda-effect} shows the impact of $\lambda$, which rules the
bias-variance trade-off in our objective function for learning the topology.
We present results for two extreme values, respectively $0.0001$ and $1000$ as
well as middle ground of $0.1$. For both datasets, $\lambda$ has little effect
on convergence speed. From a practical perspective, this is an advantage
as it removes the need for tuning $\lambda$ (one can simply set it to a
default positive value).
This behavior may be explained by the fact that reducing the bias term alone
also leads to
a reduction of variance. Hence, the variance term becomes useful
only when the bias term has been ``erased'' (or made very small), which can
happen only after a certain number of STL-FW iterations, i.e., for a potentially large $d_{\max}$. For all other experiments, we used $\lambda=0.1$.

\begin{figure}[h]
\centering

\begin{minipage}[t]{0.5\linewidth}
        \centering
         \includegraphics[height=\ratiob\linewidth]{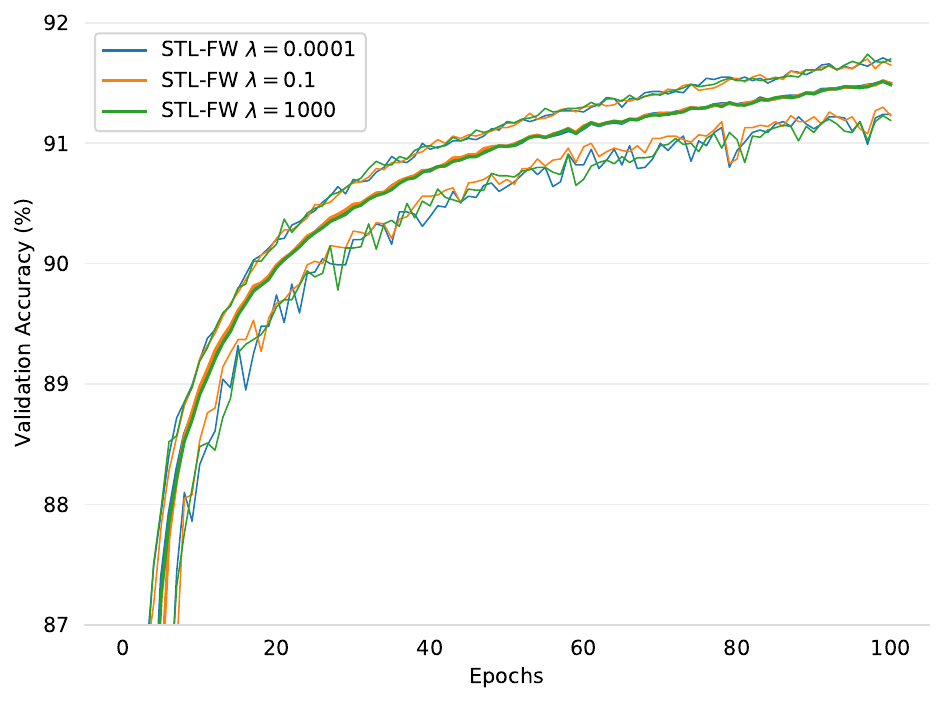}
        {\small MNIST\;}
    \end{minipage}\hfill
    \begin{minipage}[t]{0.5\linewidth}
        \centering
         \includegraphics[height=\ratiob\linewidth]{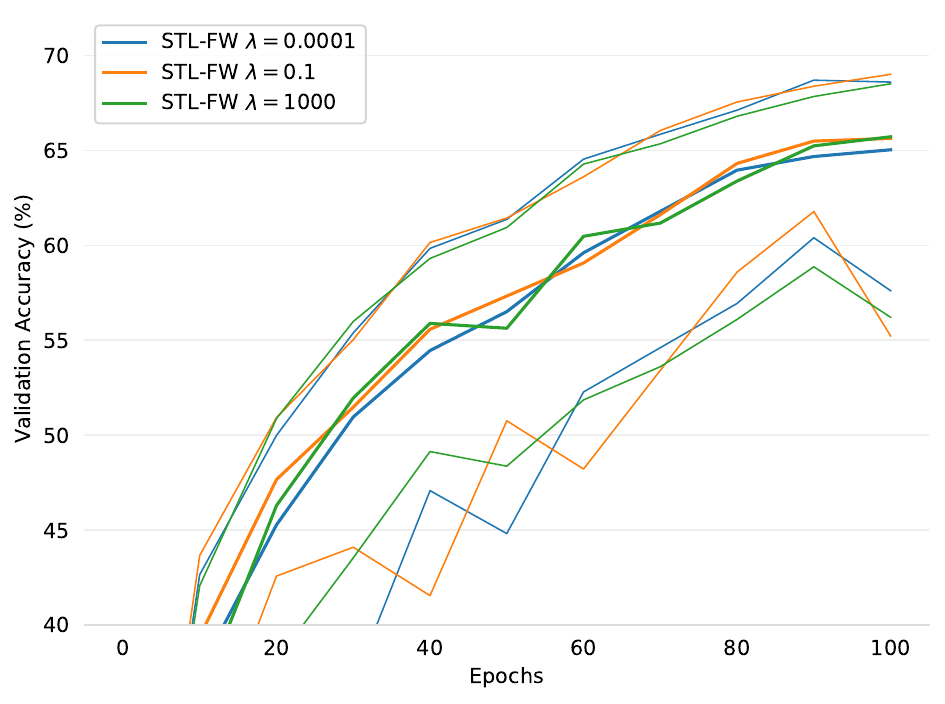}
        {\small CIFAR10\; }
    \end{minipage}

\caption{Effect of the hyperparameter $\lambda$ of STL-FW on the convergence
speed of D-SGD with 100 nodes, $d_{max}=10$.}
\label{fig:lambda-effect}
\end{figure}

\subsection{Impact of $d_{max}$ on STL-FW}

Figure~\ref{fig:effect-of-dmax-on-stl-fw} shows on a single plot the impact of
the communication budget $d_{max}$ of STL-FW on the convergence speed of
D-SGD. The communication budget has a strong impact in both cases, with STL-FW
providing the same convergence speed as fully-connected when $d_{max}=99$, but
with some residual variance because some nodes end up wth less than 99 edges.
Most of the benefits of STL-FW are obtained with the first 10 edges, with
additional edges providing only marginal benefits compared to fully-connected.
We thus chose to show all experiments of the main text with three budgets, a
small $d_{max}=2$, a medium $d_{max}=5$, and a large budget $d_{max}=10$.

\begin{figure}[h]
\centering
\begin{subfigure}{0.5\textwidth}
\centering
\includegraphics[height=\ratiob\linewidth]{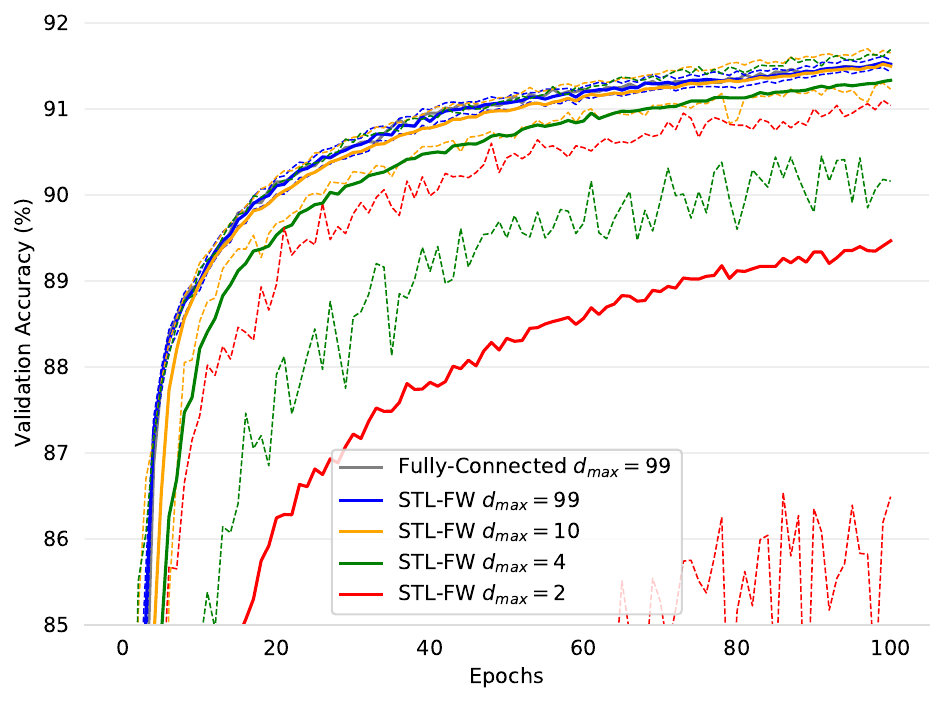}
\caption{MNIST}
\end{subfigure}\hfill
\begin{subfigure}{0.5\textwidth}
\centering
\includegraphics[height=\ratiob\linewidth]{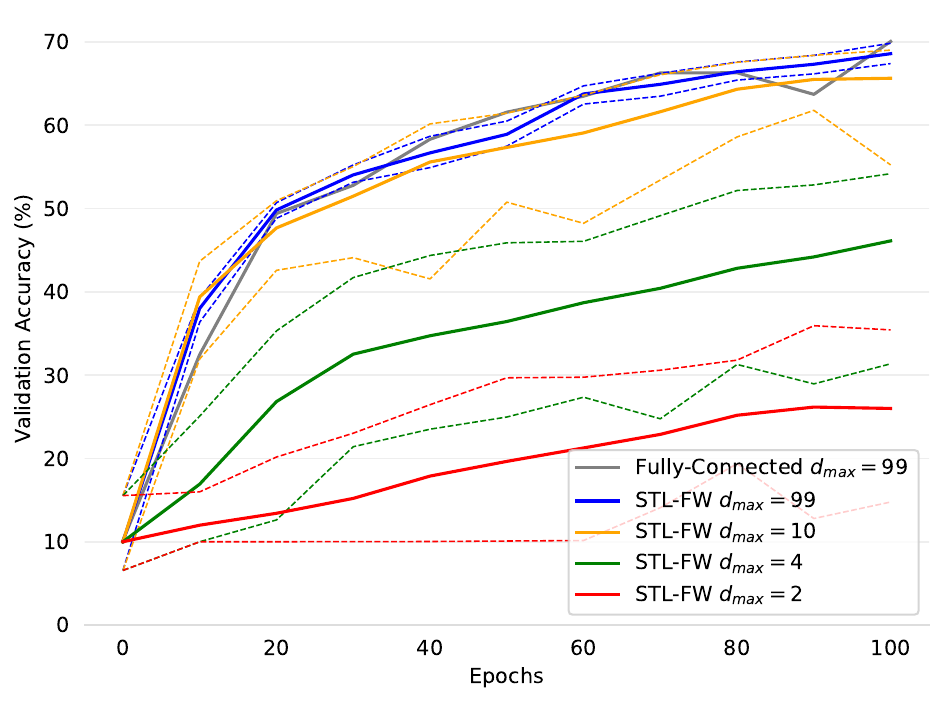}
\caption{CIFAR10}
\end{subfigure}
\caption{Effect of communication budget ($d_{max}$) of STL-FW on
the convergence speed of D-SGD with 100 nodes, $\lambda=0.1$.}
\label{fig:effect-of-dmax-on-stl-fw}
\end{figure}

Finally, for a small budget $d_{max}=2$, we had seen in
Figure~\ref{fig:real_experiments} in the main text that STL-FW did not provide
significant benefits compared to a random graph on CIFAR10.
Figure~\ref{fig:dmax-3} shows that as soon as $d_{max}=3$, STL-FW starts
providing benefits compared to a random topology on CIFAR10.

\begin{figure}[h]
\centering
\begin{subfigure}{0.5\textwidth}
\centering
\includegraphics[height=\ratiob\linewidth]{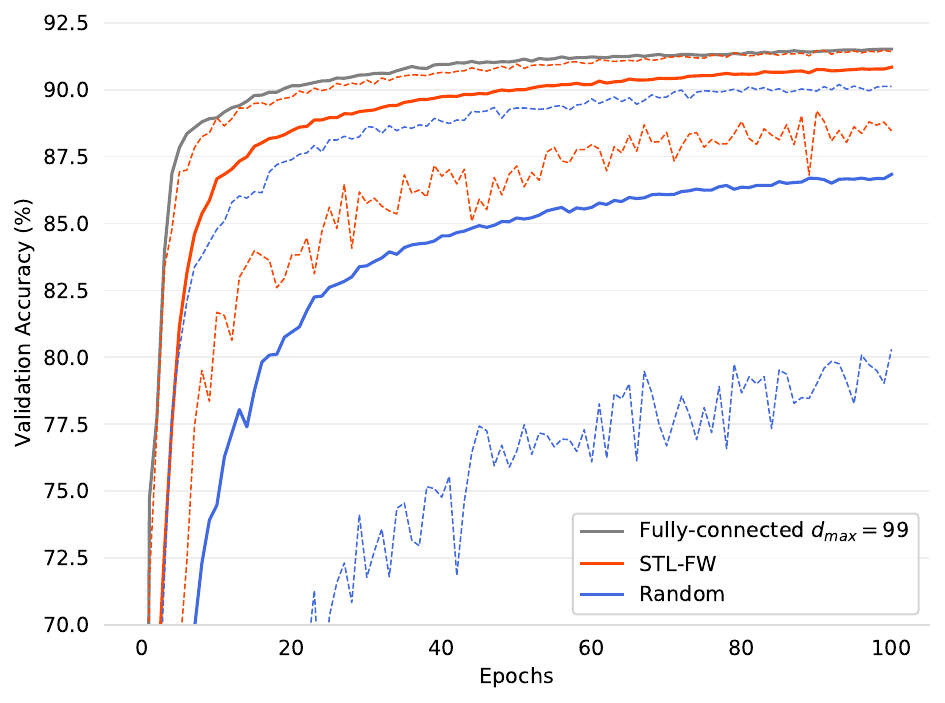}
\caption{MNIST}
\end{subfigure}\hfill
\begin{subfigure}{0.5\textwidth}
\centering
\includegraphics[height=\ratiob\linewidth]{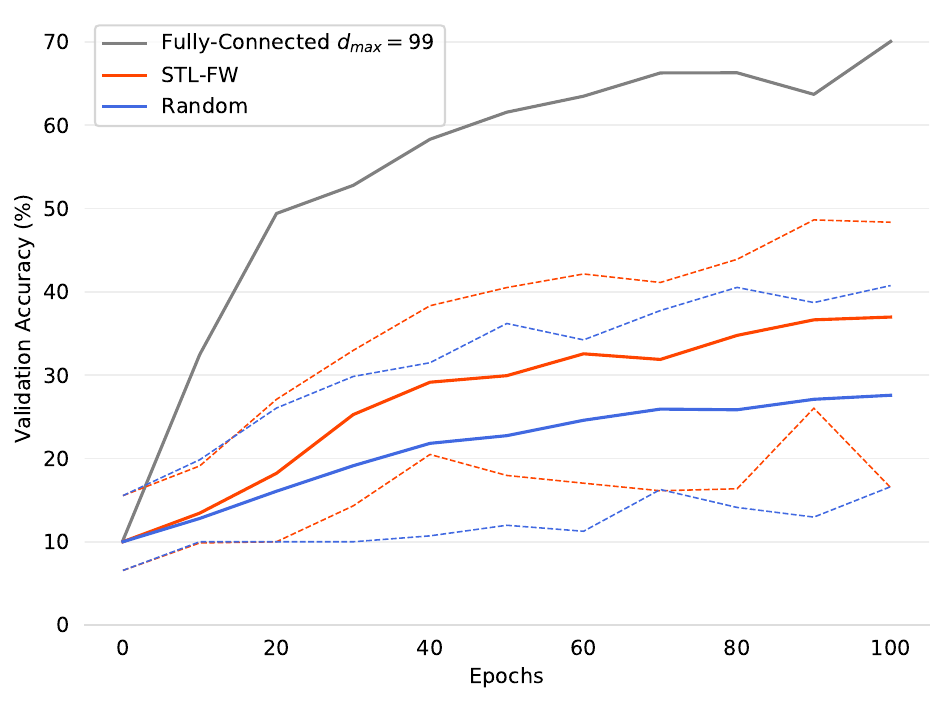}
\caption{CIFAR10  }
\end{subfigure}
\caption{Convergence speed of D-SGD with our STL-FW topology and a
random topology under small communication budget $d_{max}=3$.}
\label{fig:dmax-3}
\end{figure}

\end{document}